\documentclass[10pt,twocolumn,letterpaper]{article}

\usepackage[pagenumbers]{cvpr} % To force page numbers, e.g. for an arXiv version

\usepackage{graphicx}
\usepackage{amsmath}
\usepackage{amssymb}
\usepackage{booktabs}
\usepackage{times}
\usepackage{epsfig}
\usepackage{amsthm}
\usepackage{tabularx}
\usepackage{mathtools}
\usepackage{color, colortbl}
\usepackage{enumitem}
\usepackage{multirow}
\usepackage{hhline}
\usepackage{makecell}
\usepackage{wrapfig}
\usepackage{lineno}
\usepackage{bbm}   
\usepackage{multirow}
\usepackage[table]{xcolor}
\usepackage[accsupp]{axessibility}  % Improves PDF readability for those with disabilities.

\usepackage{amsmath,amsfonts,bm}

\def\eqref#1{equation~\ref{#1}}
\def\1{\bm{1}}

\DeclareMathAlphabet{\mathsfit}{\encodingdefault}{\sfdefault}{m}{sl}
\SetMathAlphabet{\mathsfit}{bold}{\encodingdefault}{\sfdefault}{bx}{n}

\DeclareMathOperator*{\argmax}{arg\,max}

\newcommand{\ceq}{\stackrel{\mathclap{\normalfont\mbox{c}}}{=}}
\newcommand{\cleq}{\stackrel{\mathclap{\normalfont\mbox{c}}}{\leq}}

\newcommand{\dd}{\mathbf{d}}
\newcommand{\xx}{\mathbf{x}}
\newcommand{\llll}{\mathbf{l}}
\newcommand{\pp}{\mathbf{p}}

\newcommand{\yy}{\mathbf{y}}

\newcommand{\s}{\mathbf{s}}
\newcommand{\sss}{\mathbf{s}}
\newcommand{\0}{\mathbf{0}}

\newcommand{\comment}[1]{}

\newtheorem{proposition}{Proposition}

\usepackage[pagebackref,breaklinks,colorlinks]{hyperref}

\definecolor{LightGray}{rgb}{0.9,0.9,0.9}
\definecolor{Red}{rgb}{0.9,0,0}

\usepackage[capitalize]{cleveref}
\crefname{section}{Sec.}{Secs.}
\Crefname{section}{Section}{Sections}
\Crefname{table}{Table}{Tables}
\crefname{table}{Tab.}{Tabs.}

\def\cvprPaperID{6106} % *** Enter the CVPR Paper ID here
\def\confName{CVPR}
\def\confYear{2022}

\begin{document}

%\title{Calibrating Deep Neural Networks by Constraining Logit Distance}

\title{The Devil is in the Margin: \\
Margin-based Label Smoothing for Network Calibration}

%\author{Bingyuan~Liu$^{1}$\thanks{Corresponding author: bingyuan.liu@etsmtl.ca}, Jose~Dolz$^{1}$, Adrian~Galdran$^{2}$, Riadh~Kobbi$^{3}$, Ismail~Ben~Ayed$^{1}$\\[3mm]
%$^{1}$ÉTS Montreal, Canada\hspace{0.5cm}  $^{2}$University of Bournemouth, UK  %\hspace{0.5cm} $^{3}$Diagnos Inc., Canada
%}

\author{
Bingyuan~Liu$^{1}$\thanks{Corresponding author: bingyuan.liu@etsmtl.ca}, Ismail~Ben~Ayed$^{1}$, Adrian~Galdran$^{2}$, Jose~Dolz$^{1}$\\[3mm]
$^{1}$ÉTS Montreal, Canada\hspace{0.5cm}  $^{2}$Universitat Pompeu Fabra, Barcelona, Spain  %\hspace{0.5cm} $^{3}$Diagnos Inc., Canada
}

\maketitle

%%%%%%%%% ABSTRACT
\begin{abstract}

In spite of the dominant performances of deep neural networks, recent works have shown that they are poorly calibrated, resulting in over-confident predictions. 
Miscalibration can be exacerbated by overfitting due to the minimization of the cross-entropy during training, as it promotes the predicted softmax probabilities to match the one-hot label assignments. This yields a pre-softmax activation of the correct class that is significantly larger than the remaining activations. Recent evidence from the literature suggests that loss functions that embed implicit or explicit maximization of the entropy of predictions yield state-of-the-art calibration performances.
%In this work, w
We provide a unifying constrained-optimization perspective of current state-of-the-art calibration losses. Specifically, these losses could be viewed as approximations of a linear penalty (or a Lagrangian term) imposing equality constraints on logit distances. This points to an important limitation of such underlying equality constraints, whose ensuing gradients constantly push towards a non-informative solution, which might prevent from reaching the best compromise between the discriminative performance and calibration of the model during gradient-based optimization.
Following our observations, we propose a simple and flexible generalization based on inequality constraints, which imposes a controllable margin on logit distances. Comprehensive experiments on a variety of image classification, semantic segmentation and NLP benchmarks demonstrate that our method sets novel state-of-the-art results on these tasks in terms of network calibration, without affecting the discriminative performance. The code is available at \url{https://github.com/by-liu/MbLS} .

\end{abstract}

%%%%%%%%% BODY TEXT
\section{Introduction}
\label{sec:intro}

With the advent of deep neural networks (DNNs), we have witnessed a dramatic performance improvement in a variety of computer vision and NLP tasks across different applications, such as image classification \cite{krizhevsky2012imagenet} or semantic segmentation \cite{chen2017rethinking}. Nevertheless, recent studies\cite{guo2017calibration, mukhoti2020calibrating} have shown that these high-capacity models are poorly calibrated, often resulting in over-confident predictions. As a result, the predicted probability values associated with each class overestimate the actual likelihood of %the labels from those classes being correct.
correctness.
%, suggesting that their real accuracy is likely to be lower than their predictive score. %Furthermore, %recent empirical evidence shows that this miscalibration is further exacerbated in the presence of domain shift \cite{ovadia2019can}, which can lead to potential catastrophic consequences in safety-related applications.

Quantifying the predictive uncertainty for modern DNNs has received an increased attention recently, with a variety of alternatives to better calibrate network outputs. A simple strategy consists in including a post-processing step during the test phase to transform the output of a trained network \cite{guo2017calibration, zhang2020mix, Tomani2021Posthoc, Ding2021LocalTemp}, with the parameters of this additional operation determined on a validation set. Despite their simplicity and low computational cost, these methods were shown to be effective when training and testing data are drawn from the same distribution. %\textcolor{red}{Nevertheless, under domain drift, post-hoc calibration performance largely degrades \cite{ovadia2019can}, resulting in unreliable predictions.} Furthermore, 
However, one of their observed limitations is that the choice of the transformation parameters, such as temperature scaling, is highly dependent on the dataset and network.
%these methods resort to the validation set to find empirically the optimal scaling %temperature
%parameter, which might be unrealistic in many scenarios. 
A more principled alternative is to explicitly maximize the Shannon entropy of the predictions during training 
by integrating a term into the learning objective, which penalizes confident output distributions \cite{pereyra2017regularizing}. 
Furthermore, recent efforts to quantify the quality of predictive uncertainties have focused on investigating the effect of the entropy on the training labels \cite{xie2016disturblabel,muller2019does,mukhoti2020calibrating}. Findings from these works evidence that, popular losses, which modify the hard-label assignments, such as label smoothing \cite{szegedy2016rethinking} and focal loss \cite{lin2017focal}, implicitly integrate an entropy maximization objective and have a favourable effect on model calibration. As shown comprehensively in the recent study in \cite{mukhoti2020calibrating}, these losses, with implicit or explicit maximization of the entropy, represent the state-of-the-art in model calibration.

\vspace{0.5cm}
\noindent \textbf{Contributions} are summarized as follows:

\begin{itemize}
    \item We provide a unifying constrained-optimization perspective of current state-of-the-art calibration losses. Specifically, these losses could be viewed as approximations of a linear penalty (or a Lagrangian term) imposing equality constraints on logit distances. This points to an important limitation of such underlying hard equality constraints, whose ensuing gradients constantly push towards a non-informative solution, which might prevent from reaching the best compromise between the discriminative performance and calibration of the model during gradient-based optimization.
    
    \item Following our observations, we propose a simple and flexible generalization based on inequality constraints, which imposes a controllable margin on logit distances.
    
    \item We provide comprehensive experiments and ablation studies over two standard image classification benchmarks (CIFAR-10 and Tiny-ImageNet), one fine-grained image classification dataset (CUB-200-2011), one semantic segmentation dataset (PASCAL VOC 2012) and one NLP dataset (20 Newsgroups), with various network architectures. Our empirical results demonstrate the superiority of our method compared to state-of-the-art calibration losses. Our findings suggest that, for complex datasets, such as fine-grained image classification, our margin-based method yields substantial improvements in term of calibration.

\end{itemize}
\section{Related work}
\label{sec:related}

%%% Post-processing
\noindent \textbf{Post-processing approaches.} A straightforward yet efficient strategy to mitigate mis-calibrated predictions is to include a post-processing step, which transforms the probability predictions of a deep network \cite{guo2017calibration, zhang2020mix, Tomani2021Posthoc, Ding2021LocalTemp}. Among these methods, \textit{temperature scaling}\cite{guo2017calibration}, a variant of Platt scaling \cite{platt1999probabilistic}, employs a single scalar parameter over all the pre-softmax activations, which results in softened class predictions. %While simple, this technique has demonstrated to be %has demonstrated to be effective. In particular, it employs a single scalar parameter $T>0$ over all the logits(pre-softmax activations), which softens the predicted class probabilities.
Despite its good performance on in-domain samples, \cite{ovadia2019can} demonstrated that temperature scaling does not work well under data distributional shift. \cite{Tomani2021Posthoc} mitigated this limitation by transforming the validation set before performing the post-hoc calibration step. In \cite{Ma2021postrank}, a ranking model was introduced to improve the post-processing model calibration, whereas \cite{Ding2021LocalTemp} used a simple regression model to predict the temperature parameter during the inference phase. %\textcolor{red}{Nevertheless, a main limitation of these approaches is that the parameters of the post-processing step must be determined on the validation set, which is unrealistic in practice, specially with domain shifts.}

%%% There is also Bayesian-based (probabilistic) methods and ensembles (non probabilistic)

\noindent \textbf{Probabilistic and non-probabilistic methods.} Several probabilistic and non-probabilistic approaches have been also investigated to measure the uncertainty of the predictions in deep neural networks. For example, Bayesian neural networks have been used to approximate inference by learning a posterior distribution over the network parameters, as obtaining the exact Bayesian inference is computationally intractable in deep networks. These Bayesian-based models include variational inference \cite{blundell2015weight,louizos2016structured}, stochastic expectation propagation \cite{hernandez2015probabilistic} or dropout variational inference \cite{gal2016dropout}. Ensemble learning is a popular non-parametric alternative, where the empirical variance of the network predictions is used as an approximate measure of uncertainty. This yields improved discriminative performance, as well as meaningful predictive uncertainty with reduced miscalibration. %the output of multiple models is combined leading to meaningful predictive uncertainty with reduced miscalibration. %and improved discriminative performance. 
Common strategies to generate ensembles include differences in model hyperparameters \cite{wenzel2020hyperparameter}, random initialization of the network parameters and random shuffling of the data points \cite{lakshminarayanan2016simple}, Monte-Carlo Dropout \cite{gal2016dropout,zhang2019confidence}, dataset shift \cite{ovadia2019can} or model orthogonality constraints \cite{larrazabal2021orthogonal}. However, a main drawback of this strategy stems from its high computational cost, particularly for complex models and large datasets.

\noindent \textbf{Explicit and implicit penalties.} Modern classification networks trained under the fully supervised learning paradigm resort to training labels provided as binary one-hot encoded vectors. Therefore, all the probability mass is assigned to a single class, resulting in minimum-entropy supervisory signals (i.e., entropy equal to zero). As the network is trained to follow this distribution, we are implicitly forcing it to be overconfident (i.e., to achieve a minimum entropy), thereby penalizing uncertainty in the predictions. While temperature scaling artificially increases the entropy of the predictions, \cite{pereyra2017regularizing} included into the learning objective a term to penalize confident output distributions by explicitly maximizing the entropy. In contrast to tackling overconfidence directly on the predicted probability distributions, recent works have investigated the effect of the entropy on the training labels. The authors of \cite{xie2016disturblabel} explored adding label noise as a regularization, where the disturbed label vector was generated by following a generalized Bernoulli distribution. Label smoothing \cite{szegedy2016rethinking}, which successfully improves the accuracy of deep learning models, has been shown to implicitly calibrate the learned models, as it prevents the network from assigning the full probability mass to a single class, while maintaining a reasonable distance between the logits of the ground-truth class and the other classes \cite{pereyra2017regularizing,muller2019does}. %While label smoothing softens the one-hot labels with a uniform distribution, we can also employ a teacher network to learn the soft assignment \cite{hinton2015distilling}. This family of methods, known as distillation approaches, regularize a network by adding learned information about the ratios between incorrect classes.  \textcolor{red}{Focal loss..... and then mix-up based methods.. I think with these we are already good...}
More recently, \cite{mukhoti2020calibrating} demonstrated that focal loss \cite{lin2017focal} implicitly minimizes a Kullback-Leibler (KL) divergence between the uniform distribution and the softmax predictions, thereby increasing the entropy of the predictions. Indeed, as shown in \cite{muller2019does,mukhoti2020calibrating}, both label smoothing and focal loss implicitly regularize the network output probabilities, encouraging their distribution to be close to the uniform distribution. To our knowledge, and as demonstrated experimentally in the recent studies in \cite{muller2019does,mukhoti2020calibrating}, loss functions that embed implicit or explicit maximization of the entropy of the predictions yield state-of-the-art calibration performances.

\section{Preliminaries}
\label{sec:form}

%This work is formulated in the general classification, following the optimization of a loss function for training a deep network.
%Let us denote the training dataset as $\mathcal{D}(\mathcal{X}, \mathcal{Y})=\{(\xx^{(i)}, \yy^{(i)})\}_{i=1}^N$, where $\xx^{(i)} \in \mathcal{X}$ represents the $i^{th}$ image and $\yy \in \mathcal{Y} \subset \mathbb{R}^K$ is the corresponding ground-truth label with $K$ classes. To simply the notations, the network parameters and sample indices would be omitted in the following prediction quantities and loss functions, as this does not lead to ambiguity.

%Then given a raw sample as input, the DNN yields the logit vector, denoted as $\l \in \mathbb{R}^K $ , with a soft-max layer on top to further get the probability prediction for each class, \textit{i.e.}, $s_{k} = \mathbb{S} (k| \s)$ .
%The predicted class is computed as $\hat{y} = \argmax_{y \in \mathcal{Y}} s_k$ and the predicted confidence as $\hat{p} = \max_{k} s_{k}$.

Let us denote the training dataset as $\mathcal{D}(\mathcal{X}, \mathcal{Y})=\{(\xx^{(i)}, \yy^{(i)})\}_{i=1}^N$, where $\xx^{(i)} \in \mathcal{X} \subset \mathbb{R}^{\Omega_i}$ represents the $i^{th}$ image, $\Omega_i$ the spatial image domain, and $\yy \in \mathcal{Y} \subset \mathbb{R}^K$ its corresponding ground-truth label with $K$ classes, provided as one-hot encoding. %To simplify the notations, the network parameters and sample indices would be omitted in the following sections, as this does not lead to ambiguity. 
Given an input image $\xx^{(i)}$, a neural network parameterized by $\theta$ generates a logit vector, defined as $f_{\theta}(\xx^{(i)})=\mathbf{l}^{(i)} \in \mathbb{R}^K $. To simplify the notations, we omit sample indices, as this does not lead to ambiguity, and just use $\mathbf{l} = (l_k)_{1 \leq k \leq K} \in \mathbb{R}^K$ to denote logit vectors. Note that the logits are the inputs of the softmax probability predictions of the network, which are computed
as:
\[\sss = (s_k)_{1 \leq k \leq K} \in \mathbb{R}^K; \quad s_{k} = \frac{\exp^{l_k}}{\sum_j^K \exp^{l_j}}\]
%Thus, the probability that the neural network $f_{\theta}$ predicts a given class $k$ given an input $\xx$ can be defined as $s_{k}=f_{\theta}(k|\xx)$. The predicted class is therefore computed as $\hat{y} = \argmax_{y \in \mathcal{Y}} s_k$ and the predicted confidence as $\hat{p} = \max_{y \in \mathcal{Y}} s_{k}$.
The predicted class is computed as $\hat{y} = \argmax_k s_k$, whereas the predicted confidence is given by $\hat{p} = \max_{k} s_k$.

\noindent \textbf{Calibrated models.} %We would like the predicted confidence of the model to be calibrated, which intuitively means that it represents a true probability. 
\textit{Perfectly calibrated} models are those for which the predicted confidence for each sample is equal to the model accuracy : $\hat{p} = \mathbb{P}(\hat{y} = y | \hat{p})$, where $y$ denotes the true labels.
Therefore, an \textit{over-confident model} tends to yield predicted confidences that are larger than its accuracy, whereas an \textit{under-confident model} displays lower confidence than the model's accuracy.

%Therefore, an \textit{over-confident model} represents that its predicted confidence tends to be larger than the model accuracy, while an \textit{under-confident model} means relatively smaller confidence than the ideal value.

% The primary metric used to measure model calibration is the \textit{expected calibration error} (ECE) \cite{naeini2015ece}, which is the expected absolute difference between the predicted confidence and accuracy of a particular model : $ \mathbb{E}_{\pp} [|\mathbb{P}(\hat{\yy} = \yy | \hat{p}) - \hat{p} |]$.
% We also use AdaECE and Classwise-ECE in our evaluation.
% Please refer to the supplement material for detailed presentation of the three metrics.

%\noindent \textbf{Mis-calibration of DNNs:} The mis-calibration of high-capacity DNNs, in particular the over-confidence, is primarily due to the over-fitting on the conventional cross-entropy (CE) loss : 

%\begin{align}\label{eq:ce}
%    {\cal L}_\text{CE} = -\sum_{k} y_{k} \log s_{k}
%\end{align}

%It is noted that the minimization of CE is achieved when the correct node of softmax prediction is equal to $1$ with the rest as $0$ for every training sample. Therefore, the mini-batch optimization pushes the prediction to this simplex irrespective of the classification accuracy.

\noindent \textbf{Miscalibration of DNNs.} The cross-entropy (CE) loss is the standard
training objective for fully supervised discriminative deep models. CE reaches its minimum when the predictions for all the training samples match the hard (binary) ground-truth labels, i.e., $s_k = 1$ when $k$ is the ground-truth class of the sample and $s_k = 0$ otherwise. Minimizing the CE implicitly pushes softmax vectors $\sss$ towards the vertices of the simplex, thereby magnifying the distances between the largest logit $\max_k(l_k)$ and the rest of the logits,   
%For this to happen, $|| \mathbf{W}_{\theta}|| \rightarrow \infty$ (the norms of the weights), which results in high logit values, and thus
yielding over-confident predictions and miscalibrated models.  
%In other words, CE implicitly induces a weight magnification effect in neural networks, leading to higher logit distances and, therefore, miscalibrated models.

%\noindent \textbf{Definition of logit distances:} Let us now define the vector of logit distances between the winner class and the rest as:

%\begin{align}\label{eq:logit distance}
%    \dd(\llll) = (\max_i (l_i) - l_k)_{k \neq \argmax(\llll)} \in \mathbb{R}^{K-1} 
%\end{align}

%where $\max_i(l_i)$ presents the largest value of the logit vector and $k$ is not equal to the index corresponding to the winner node.
%In the following section, we also simply this notation as $d_k = \max (\llll) - l_k$.
%Noted that each element in $\dd(\llll)$ is non-negative.
%According to the definition of softmax function : $s_{k} = \frac{e^{l_k}}{\sum_k e^{l^k}}$, it is noted that the goal of CE is possible only when the logit distance $\dd(\llll)$ is large enough.
%In other words, CE pushes logit distance to be infinite, \textit{i.e.}, $\dd(\llll) \rightarrow \infty$.

\section{A constrained-optimization perspective of calibration}
\label{sec:view}

In this section, we present a novel constrained-optimization perspective of current calibration methods for deep networks, showing that the existing 
strategies, including Label Smoothing (LS) \cite{muller2019does,szegedy2016rethinking}, Focal Loss (FL) \cite{mukhoti2020calibrating,lin2017focal} and Explicit Confidence Penalty (ECP) \cite{pereyra2017regularizing}, impose {\em equality} constraints on logit distances. Specifically, they embed either explicit or implicit penalty functions, which push all the logit distances to zero.

\subsection{Definition of logit distances}
Let us first define the vector of logit distances between the winner class and the rest as:
\begin{align}\label{eq:logit distance}
    \dd (\llll) = (\max_j (l_j) - l_k)_{1 \leq k \leq K} \in \mathbb{R}^{K} 
\end{align}
Note that each element in $\dd(\llll)$ is non-negative.
In the following, we show that LS, FL and ECP correspond to different {\em soft penalty} functions for imposing the same hard equality constraint $\dd (\llll) = \mathbf 0$ or, equivalently, imposing inequality 
constraint $\dd (\llll) \leq \mathbf 0$ (as $\dd (\llll)$ is non-negative by definition). 
Clearly, enforcing this equality constraint in a hard manner would result in all $K$ logits being equal for a given sample, which corresponds to non-informative softmax predictions $s_k = \frac{1}{K} \, \forall k$.  

\subsection{Penalty functions in constrained optimization}

In the general context of constrained optimization\cite{Bertsekas95}, {\em soft} penalty functions are widely used to tackle {\em hard} equality or inequality constraints. For the discussion in the sequel, consider specifically the following hard equality constraint:
\begin{equation}
\label{equality-constraint-label-smoothing}
\dd (\llll) = {\mathbf 0} 
\end{equation}
The general principle of a soft-penalty optimizer is to replace a hard constraint of the form in Eq.~\ref{equality-constraint-label-smoothing} by adding an additional term $\mathcal{P}(\dd (\llll))$ into the main objective function to be minimized. Soft penalty $\mathcal{P}$ should be a continuous and differentiable function, which reaches its global minimum when the constraint is satisfied, i.e., it verifies: $\mathcal{P}(\dd (\llll)) \geq \mathcal{P}(\mathbf {0}) \, \forall \, \llll \in \mathbb{R}^{K}$.
Thus, when the constraint is violated, i.e., when $\dd (\llll)$ deviates from $\mathbf {0}$, the penalty term $\mathcal{P}$ increases.

\noindent \textbf{Label smoothing.} In addition to improving the discriminative performance of deep neural networks, recent evidence \cite{lukasik2020does,muller2019does} suggests that Label Smoothing (LS) \cite{szegedy2016rethinking} positively impacts model calibration. In particular, LS modifies the hard target labels with a smoothing parameter $\alpha$, so that the original one-hot training labels $\yy \in \{0, 1\}^K$ become $\yy^\text{LS} = (y_{k}^\text{LS})_{1 \leq k \leq K}$, with $y_{k}^\text{LS}=y_{k}(1-\alpha)+\frac{\alpha}{K}$. Then, we simply minimize the cross-entropy between the modified labels and the network outputs:
\begin{align}
\label{eq:ls}
    {\cal L}_\text{LS} = -\sum_{k} y_{k}^\text{LS} \log s_{k} = -\sum_{k} ((1-\alpha)y_k + \frac{\alpha}{K}) \log s_{k}
\end{align}
where $\alpha \in [0,1]$ is the smoothing hyper-parameter.
It is straightforward to verify that cross-entropy with label smoothing in Eq.~\ref{eq:ls} can be decomposed into a 
standard cross-entropy term augmented with a Kullback-Leibler (KL) divergence between uniform distribution ${\mathbf u} = \frac{1}{K}$ and the softmax prediction:
\begin{align}
\label{eq:ls-kl}
{\cal L}_\text{LS} \ceq {\cal L}_\text{CE} + \frac{\alpha}{1-\alpha}{\cal D}_\text{KL}\left({\mathbf u} || \s \right )
\end{align}
where $\ceq$ stands for equality up to additive and/or non-negative multiplicative constants. Now, consider the following bounding relationships between a linear penalty (or a Lagrangian) for equality constraint $\dd (\llll) = {\mathbf 0}$ and the KL divergence in Eq.~\ref{eq:ls-kl}.  
\begin{proposition}
\label{prop:ls}
A linear penalty (or a Lagrangian term) for constraint $\dd (\llll) = {\mathbf 0}$ is bounded from above and below
by ${\cal D}_\text{KL}\left({\mathbf u} || \s \right )$, up to additive constants:
\begin{align}
{\cal D}_\text{KL}\left({\mathbf u} || \s \right ) - \log(K) \cleq \frac{1}{K}\sum_k (\max_j (l_j) - l_k) \cleq  {\cal D}_\text{KL}\left({\mathbf u} || \s \right ) \nonumber
\end{align}
where $\cleq$ stands for inequality up to an additive constant.
\end{proposition}
These bounding relationships could be obtained directly from the softmax and ${\cal D}_\text{KL}$
expressions, along with the following well-known property of the LogSumExp function: $\max_k(l_k) \leq \log{\sum_k^K e^{l_k}} \leq \max_k(l_k) + \log(K)$.
For the details of the proof, please refer to Appendix~\ref{sec:ap:proof}.

Prop. \ref{prop:ls} means that LS is (approximately) optimizing a linear penalty (or a Lagrangian) for logit-distance constraint $\dd(\llll) = \0$, which encourages equality of all logits; see the illustration in Figure~\ref{fig:method}, top-left.  

\noindent \textbf{Focal loss.} Another popular alternative for calibration is focal loss (FL) \cite{lin2017focal}, which attempts to alleviate the over-fitting issue in CE by directing the training attention towards samples with low confidence in each mini-batch. More concretely, the authors 
proposed to use a modulating factor to the CE, $(1-s_k)^{\gamma}$, which controls the trade-off between easy and hard examples. Very recently, \cite{mukhoti2020calibrating} demonstrated that focal loss is, in fact, an upper bound on CE augmented with a term that implicitly serves as a maximum-entropy regularizer:
\begin{align}
{\cal L}_\text{FL} = -\sum_{k} (1 - s_{k})^{\gamma} y_{k} \log s_{k} \geq \mathcal{L}_{\text{CE}} - \gamma\mathcal{H}(\mathbf{s})
\label{eq:fl}
\end{align}
where $\gamma$ is a hyper-parameter and $\mathcal{H}$ denotes the Shannon entropy of the softmax prediction, given by
\[\mathcal{H}(\mathbf{s}) = -\sum_k s_k \log(s_k)\]
In this connection, FL is closely related to ECP \cite{pereyra2017regularizing}, which explicitly added the negative entropy term, $-\mathcal{H}(\mathbf{s})$, to the training objective. It is worth noting that minimizing the negative entropy of the prediction is equivalent to minimizing the KL divergence between the prediction and the uniform distribution, up to an additive constant, i.e., 
\[-\mathcal{H}(\mathbf{s}) \ceq {\cal D}_\text{KL}(\s || \mathbf{u})\] 
which is a reversed form of the KL term in Eq.~\ref{eq:ls-kl}. 

Therefore, all in all, and following Prop.~\ref{prop:ls} and the discussions above, %both
LS, FL and ECP could be viewed as different penalty functions for imposing the same logit-distance equality constraint $\dd(\llll) = \0$. This motivates our margin-based generalization of logit-distance constraints, which we introduce in the following section, along with discussions of its desirable properties (e.g., gradient dynamics) for calibrating neural networks. 

\begin{figure}[h]
    \centering
    \includegraphics[width=1\columnwidth]{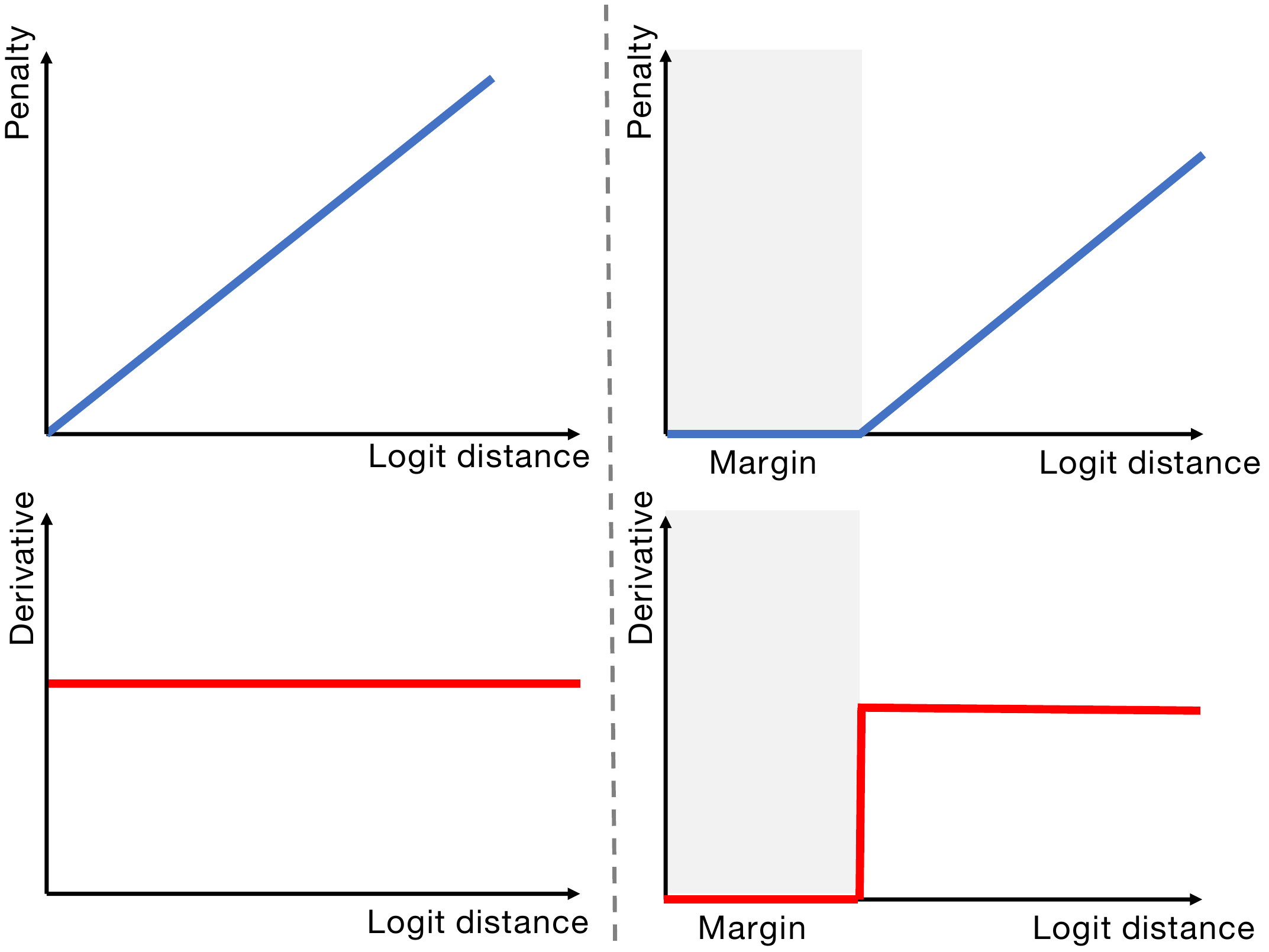}
    \caption{Illustration of the linear (left) and margin-based (right) penalties for imposing 
    logit-distance constraints, along with the corresponding derivatives.}
    \label{fig:method}
    \vspace{-2mm}
\end{figure}

\subsection{Margin-based Label Smoothing (MbLS)}
\label{sec:our}

Our previous analysis shows that LS, FL and ECP are closely related from a constrained-optimization perspective, and they could be seen as approximations of a linear penalty for imposing constraint $\dd (\llll) = {\mathbf 0}$, pushing all logit distances to zero; see Figure~\ref{fig:method}, top-left. Clearly, enforcing this constraint in a hard way yields a non-informative solution where all the classes have exactly the same logit and, hence, the same class prediction: $s_k = \frac{1}{K}\, \forall K$. While this trivial solution is not reached in practice when using soft penalties (as in LS, FL and ECP) jointly with CE, we argue that the underlying equality constraint $\dd (\llll) = {\mathbf 0}$ has an important limitation, which might prevent from reaching the best compromise between the discriminative performance and calibration of the model during gradient-based optimization. Figure \ref{fig:method}, left, illustrates this: With the linear penalty for constraint $\dd (\llll) = {\mathbf 0}$ in the top-left of the Figure, the 
derivative with respect to logit distances is a strictly positive constant (left-bottom), yielding during training {\em a gradient term that constantly pushes towards the trivial, non-informative solution} $\dd (\llll) = {\mathbf 0}$ (or equivalently $s_k = \frac{1}{K}\, \forall K$). To alleviate this issue, we propose to replace the
equality constraint $\dd (\llll) = {\mathbf 0}$ with the more general inequality constraint $\dd (\llll) \leq {\mathbf m} $, where ${\mathbf m}$ denotes the $K$-dimensional vector with all elements equal to $m > 0$.
Therefore, we include a margin $m$ into the penalty, so that the logit distances in $\dd(\llll)$ are allowed to 
be below $m$ when optimizing the main learning objective:
\begin{align}
\label{eq:our-constraint}
    \min \quad   \mathcal{L}_{\text{CE}} \quad \text{s.t.} \quad  \dd(\llll) \leq \textbf{m} , \quad \textbf{m} > \textbf{0}
\end{align}
The intuition behind adding a strictly positive margin $m$ is that, unlike the linear penalty for constraint $\dd (\llll) = {\mathbf 0}$ (Figure \ref{fig:method}, left), the gradient is back-propagated only on those logits where the distance is above the margin (Figure \ref{fig:method}, right). This contrasts with the 
linear penalty, for which there exists always a gradient, and its value is the same across all the logits, regardless of their distance.

Even though the constrained problem in Eq. \ref{eq:our-constraint} could be solved by a Lagrangian-multiplier algorithm, we resort to a simpler unconstrained approximation by ReLU function:
\begin{align}
\label{eq:our-l1}
    \min \quad &  \mathcal{L}_{\text{CE}} + \lambda \sum_k \max(0, \max_j (l_j) - l_k - m)
\end{align}
Here, the non-linear ReLU penalty for inequality constraint $\dd (\llll) \leq {\mathbf m}$ discourages logit distances from surpassing a given margin $m$, and $\lambda$ is a trade-off weight balancing the two terms.
%Clearly, 
It is clear that, as discussed in Sec.~\ref{sec:view}, several competitive calibration methods could be viewed as approximations for imposing constraint $\dd (\llll) = {\mathbf 0}$ and, therefore, correspond to the special case of our method when setting the margin to $m=0$. Our comprehensive experiments in the next section demonstrate clearly the benefits of introducing a strictly positive margin $m$.

Note that our model in Eq.~\ref{eq:our-l1} has two hyper-parameters, $m$ and $\lambda$. We fixed $\lambda$ to $0.1$ in all our experiments on a variety of problems and benchmarks, and tuned only the margin $m$ over validation sets. In this way, when comparing with the existing calibration solutions, we use the same budget of hyper-parameter optimization ($m$ in our method vs. $\alpha$ in LS or $\gamma$ in FL).

%Obviously the implicit constraint of $m=0$ in previous methods is not ideal.
%Our solution tackles this issue by providing a way to introduce a reasonable  and controllable margin.
%By this way, we are able to push the over-confident DNNs toward perfect calibrated models with less risk of decreasing the classification performance.
%Meanwhile, the ${\cal L}_1$ penalty we used is able to provide a more stable gradient property at the vicinity than the logarithmic function hinted in LS (Eq. \ref{eq:ls-kl}) and the polynomial function in FL (Eq. \ref{eq:fl}).

 \begin{table*}[h!]
  \caption{Calibration performance for different approaches on two popular image classification benchmarks. %Values of \textit{ECE} (\%) and \textit{AECE} (\%) for each method are reported. 
  Two models are used on each dataset : ResNet-50 (R-50) and ResNet-101 (R-101).
  Best method is highlighted in bold, whereas the second best method is underlined. %The fine-grained image classification dataset is shadowed in the table.%
  }
  \vspace{-2mm}
  \label{table:big}
  \centering
  %\small \vspace{1em}
   \footnotesize
  \resizebox{0.9\textwidth}{!}
  {
  \setlength{\tabcolsep}{2.0pt}
  \begin{tabular}{@{}llllccccccccccccccccccccc@{}}
    \toprule
    %\multicolumn{2}{c}{Part}                   \\
    % \cmidrule(r){1-2}
    \multirow{2}{*}{\textbf{Dataset}}  && \multirow{2}{*}{\textbf{Model}}  && \multicolumn{2}{c}{\textbf{CE}} &&
    \multicolumn{2}{c}{\textbf{ECP}\cite{pereyra2017regularizing}} &&
			\multicolumn{2}{c}{\textbf{LS} \cite{szegedy2016rethinking}} && \multicolumn{2}{c}{\textbf{FL} \cite{lin2017focal}} &&
			\multicolumn{2}{c}{\textbf{FLSD}\cite{mukhoti2020calibrating}} && 
			\multicolumn{2}{c}{\textbf{Ours (m=0)}} && \multicolumn{2}{c}{\textbf{Ours}} \\
    \cmidrule{5-6} \cmidrule{8-9} \cmidrule{11-12} 
    \cmidrule{14-15} \cmidrule{17-18} \cmidrule{20-21} \cmidrule{23-24}
    &&  && ECE & AECE && ECE & AECE && ECE & AECE && ECE & AECE && ECE & AECE && ECE & AECE && ECE & AECE \\
    \midrule
    \multirow{2}{*}{Tiny-ImageNet} && R-50 && 3.73 & 3.69  && 4.00 & 3.92 && 3.17 & 3.16 && 2.96 & 3.12 && 2.91 & 2.95 && \underline{2.50} & \underline{2.58} && \textbf{1.64} & \textbf{1.73} \\
    && R-101 && 4.97 & 4.97 && 4.68 & 4.66 && 2.20 & 2.21 && 2.55 & 2.44 && 4.91 & 4.91 && \underline{1.89} & \underline{1.95} && \textbf{1.62} & \textbf{1.68} \\
    \midrule
    \multirow{2}{*}{CIFAR-10} && R-50 && 5.85 & 5.84 && 3.01 & \textbf{2.99} && \underline{2.79} & 3.85 && 3.90 & 3.86 && 3.84 & 3.60 && 3.72 & 4.29 && \textbf{1.16} & \underline{3.18} \\
    && R-101 && 5.74 & 5.73 && 5.41 & 5.40 && 3.56 & 4.68 && 4.60 & 4.58 && 4.58 & 4.57 && \underline{3.07} & \underline{3.97} && \textbf{1.38} & \textbf{3.25}  \\
    % \midrule
    % \rowcolor{LightGray}
    %  && R-50 && 6.67 & 6.43 && 6.32 & 6.30 &&  \underline{3.63} & 3.79 && 4.03 & 4.09 && 3.68 & \underline{3.73} && 7.388 & 7.389 && \textbf{2.68}  & \textbf{2.67} \\
    % \rowcolor{LightGray}
    % \multirow{-2}{*}{CUB-200} && R-101 && 6.75 & 6.65 && 5.55 & 5.44 && 5.16 & \underline{5.14} && 8.41 & 8.39 && 8.54 & 8.53 && \underline{5.11} & 5.29 && \textbf{2.78} & \textbf{2.63} \\
    \bottomrule
  \end{tabular}
  }
  \vspace{-3mm}
\end{table*}

\section{Experiments}
\label{sec:exp}

\noindent \textbf{Datasets.} Our method is validated on a variety of popular image classification benchmarks, including two standard datasets, \textbf{CIFAR-10} \cite{krizhevsky2009learning}
% CIFAR-100
and \textbf{Tiny-ImageNet} \cite{deng2009imagenet}, and one fine-grained dataset, \textbf{CUB-200-2011} \cite{WahCUB_200_2011}. A main difference between these tasks is that fine-grained visual categorization focuses on differentiating between \textit{hard-to-distinguish} object classes, typically from subcategories, such as species of birds or flowers, whereas conventional datasets contain more general categories, i.e., \textit{is this a dog or a car?} To show the general applicability of our method, we also evaluate it on one well-known segmentation benchmark, \textbf{PASCAL VOC 2012} \cite{VOC2015}. %For classification tasks, state-of-the-art architectures such as ResNet \cite{he2016deep}, DenseNet \cite{huang2017densely} and WideReseNet \cite{zagoruyko2016wide} are employed, whereas DeepLabV3 \cite{chen2017rethinking} is used for semantic segmentation. 
Last, we conduct experiments on the \textbf{20 Newsgroups} dataset \cite{lang1995newsweeder}, a popular Natural Language Processing (NLP) benchmark for text classification. 
Please refer to Appendix~\ref{sec:ap:data} for a detailed description of each dataset.
%and experimental settings.
%For this task, we train a Global Pooling CNN (GPool-CNN) \cite{lin2013network} architecture following \cite{mukhoti2020calibrating}. For a fair comparison, we employ the same settings across all the benchmarks and models. Please refer to the supplementary material for a detailed description of each dataset and training settings.

\noindent \textbf{Architectures.} We used ResNet \cite{he2016deep} for the image classification tasks, and DeepLabV3 \cite{chen2017rethinking} for semantic segmentation. Regarding the NLP recognition task, we train a Global Pooling CNN (GPool-CNN)  architecture \cite{lin2013network}, following \cite{mukhoti2020calibrating}. For a fair comparison, we employ the same settings across all the benchmarks and models. We refer the reader to Appendix~\ref{sec:ap:data} for a detailed description of the training settings.

\noindent \textbf{Metrics.} To evaluate the calibration performance, we resort to the standard metric in the literature \cite{mukhoti2020calibrating}: expected calibration error (ECE) \cite{naeini2015ece}. This metric represents the expected absolute difference between the predicted confidence and model accuracy: $ \mathbb{E}_{\pp} [|\mathbb{P}(\hat{\yy} = \yy | \hat{p}) - \hat{p} |]$.
In practice, an approximate estimation is used to calculate ECE given a finite number of samples. Specifically, we group the predictions into $M$ equispaced bins. Let $B_{m}$ denote the set of samples with predicted confidence belonging to the $m^{th}$ bin, where the interval is $[\frac{i-1}{M}, \frac{i}{M}]$.
Then, the accuracy of $B_{m}$ is: $A_{m} = \frac{1}{|B_m|} \sum_{i \in B_m} \mathbbm{1}(\hat{\yy}_i = \yy_i )$, where $\mathbbm{1}$ is the indicator function.
Similarly, the mean confidence of $B_m$ is defined as the average confidence of all samples in the bin : $C_{m} = \frac{1}{|B_m|} \sum_{i \in B_m} \hat{p}_i$.
Then, ECE can be approximated as a weighted average of the absolute difference between the accuracy and confidence of each bin:

\vspace{-3mm}

\begin{align}
\label{eq:ece}
%   1 - Dice  \leq -log(Dice) \leq K\cdot CE + Cst
  ECE =  \sum_{m}^{M} \frac{|B_m|}{N} |A_m - C_m|
\end{align}
In our implementation, the number of bins is set to $M=15$. We also consider Adaptive ECE (AECE) for which bin sizes are calculated to evenly distribute samples across the bins.

%The widely used visualising calibration, \textit{reliability diagrams}\cite{mizil2015predictinggood}, is also used in this paper, which plots the accuracy of each bin as a function of its confidence.
%For a perfectly calibrated model, it is a diagonal line as the accuracy of each bin ideally matches the corresponding confidence.
%By contrast, the curve mostly lies above the diagonal for an under-confident model, while an over-confident model would show a curve mostly below the diagonal.

% In addition, we show the widely employed \textit{reliability diagrams} \cite{mizil2015predictinggood}, which plot the accuracy as a function of the confidence. A perfectly calibrated model has a reliability diagram that approximates a diagonal line, since the accuracy of each bin ideally matches the corresponding confidence. In contrast, the curve mostly lies above the diagonal for an under-confident model, while an over-confident model would show a curve mostly below the diagonal.
To measure the discriminative performance of classification models, we provide the accuracy (Acc) on the testing set. Finally, the mean intersection over union (mIoU) is employed to measure the segmentation performance.

%and Top-5 accuracy (Acc@5).
% Furthermore, Top-5 accuracy (Acc@5) and Top-5 error (Err@5) are also included as better calibrated model should present better scores on these metrics.
%On the other hand, the mean intersection over union (mIoU) is employed to measure the segmentation performance.

\noindent \textbf{Baselines.} In addition to cross-entropy (CE), we evaluate the performance of relevant works, including label smoothing (LS) \cite{szegedy2016rethinking}, focal loss (FL) \cite{lin2017focal} and explicit confidence penalty (ECP) in \cite{pereyra2017regularizing}. In addition, we also include the results from the recent adaptive sample-dependent focal loss (FLSD) in \cite{mukhoti2020calibrating}, which provided highly competitive calibration performances and advocated the use of FL for calibration\footnote{In fact, initially designed for object detection, FL was not used for calibration before the recent study in \cite{mukhoti2020calibrating}.}. To set the hyper-parameters of the different methods, we employed the values reported %follow the best practices 
in recent literature \cite{muller2019does,mukhoti2020calibrating}. More precisely, the smoothing factor $\alpha$ in LS is set to $0.05$, $\gamma$ in FL is set to $3$, and the scheduled $\gamma$ in FLSD is $5$ for $s_k \in [0, 0.2)$ and $3$ for $s_k \in [0.2, 1)$ (with $k$ being the right class for a given sample).
Last, we empirically set the balancing hyper-parameter in ECP to $0.1$, as it brings consistent performance in our experiments. 

%In addition to cross-entropy (CE), other related methods include including label smoothing (LS), focal loss(FL) and negative entropy prediction penalty (ECP) \cite{pereyra2017regularizing}. We also compare with adaptive sample-dependent focal loss (FLSD) proposed in \cite{mukhoti2020calibrating}. Regarding the hyper-parameters of LS, FL and FLSD, we follow the best practices in \cite{guo2017calibration}. Specially, the smoothing factor $\alpha$ in LS is set to $0.05$, $\lambda$ in FL is set to $3$, and the scheduled $\lambda$ in FLSD is $5$ for $s_k \in [0, 0.2)$ and $3$ for $s_k \in [0.2, 1)$ ($k$ is the right class for the sample). We empirically set the trade-off hyper-parameters in ECP to $0.1$ , as it brings consistent performance in our experiments. 

\noindent \textbf{Our method.} The proposed method has only one hyper-parameter $m$ (we kept $\lambda$ fixed to $0.1$, so that the label-smoothing term has the same budget of hyper-parameters as the other methods). 
%During our experiments, we empirically found that fixing $\lambda = 0.1$ yielded consistent performance, so we kept this value for all the experiments with margins. 
As for margin $m$, it was chosen based on the validation set of each dataset, which yielded relatively stable margin values across different tasks and consistent behaviour over both validation and testing data (See Figure \ref{fig:tiny-margin}): $m=6$ on CIFAR-10 and 20 Newsgroup, and $m=10$ on Tiny-Imagenet, CUB-200-2011 and Pascal VOC Segmentation.
% \textcolor{red}{and XXXX for the 20 newsgroup dataset....}.
Note that we perform ablation studies to assess the impact of varying $m$.

%The hyper-parameters introduced by our method are the constrain of logit distance $m$ and the trade-off weight $\lambda$. In our experiments, we empirically find that it can yield consistent performance with fixed $\lambda = 0.1$ and we tune only $m$. However, the best value of $m$ is also relatively stable for each task : $6$ on CIFAR-10/100, $10$ on Tiny-Imagenet, CUB-200-2011 and Pascal VOC Segmentation.
%We also study the effect of different $m$ and show how it affect the results.

%\subsection{Results on image classification. } 
\subsection{Results} 
\label{ssec:quant}

%In Table \ref{table:big}, we report ECE (\%) along with Acc (\%) on the test set of four image classification datasets with different model architectures. The lower the two scores, the better the model. Overall, this table clearly shows that the proposed loss outperform all the other baseline, which delivers the best calibration performance in general.  Out of the total $11$ experimental settings, it yields the lowest ECE in $7$ cases and the second lowest in $2$ settings. With the effect of the introduced penalty, our solution increase the baseline CE by a large margin. For example, it brings $71.9\%$ relative improvement on CIFAR-10 using ResNet-50, while the improvement is $55.9\%$ on Tiny-ImageNet using and $70.8$ on CUB-200-2011 using ResNet-50. Regarding the Err, the results of our method are comparable, stable and even better in many cases. Nevertheless, the test error of other related losses is unstable mainly due to the side effect of its aggressive target of pushing logit distance to $0$ and unbounded gradient at the vicinity, like LS under CIFAR-10 ResNet-50 and CIFAR-10 ResNet110, FL under CIFAR-10 DenseNet-121 and Tiny-ImageNet ResNet-101.

\noindent \textbf{Standard image classification benchmarks.} We first evaluate the calibration behaviour of both baselines and proposed model on two well-known image classification datasets, whose results are reported in Table \ref{table:big}. In particular, we show that training a model with hard targets, i.e., CE, leads to miscalibrated predictions across datasets and backbone architectures. In addition, by penalizing low-entropy predictions, either explicitly (i.e., ECP \cite{pereyra2017regularizing}) or implicitly (i.e., LS \cite{szegedy2016rethinking}, FL \cite{lin2017focal} and FLSD \cite{mukhoti2020calibrating}), we can typically train better calibrated networks. %, at least on more traditional datasets. 
Intuitively, the regularization terms added by these methods interplay with the main cross-entropy objective, controlling up to some extent the amount of confidence on the predictions. %predictive estimates. 
Thus, even though the impact of the different methods differs across datasets, the calibration performance is typically improved over the standard cross entropy training. Last, we can observe that both versions of our model yield the best results in nearly all of the cases, with just one setting ranking second across all the models. Furthermore, the significant improvement observed when the margin is included, i.e., $m>0$, motivates its use, suggesting that our method provides better calibrated networks. {\em An interesting observation is that, while existing models are quite sensitive to the employed backbone, the predicted uncertainty estimates provided by our models are considerably robust, presenting the smallest variations across architectures}. For instance, when using higher-capacity backbones on CIFAR-10, calibration metrics across all existing methods are considerably degraded (ECP \cite{pereyra2017regularizing}:+2.4, LS \cite{szegedy2016rethinking}:+0.77, FL \cite{lin2017focal}:+0.7 and FLSD: \cite{mukhoti2020calibrating}:+0.74), whereas models calibrated with our approach suffer minor changes (\textit{Ours}:+0.22).
For the reliability diagrams of the models, please refer to Appendix~\ref{sec:ap:diagram}.

%\textcolor{red}{Jose: If there is space left, add here some explanation/intuition about why this. LS could be more aggressive by enforcing zero/smaller margins, whereas the fact of allowing margins in our method might help on the gradients. Look at what we wrote in the motivation and link with it....} \textcolor{red}{Nevertheless, the test error of other related losses is unstable mainly due to the side effect of its aggressive target of pushing logit distance to $0$ and unbounded gradient at the vicinity, like LS under CIFAR-10 ResNet-50 and CIFAR-10 ResNet110, FL under CIFAR-10 DenseNet-121 and Tiny-ImageNet ResNet-101.}

In terms of discriminative performance (Table~\ref{table:bigClass}), we can observe that, on the one hand, MbLS yields performances on par with LS and CE, sometimes ranking as the best method. On the other hand, FL and its variant FLSD obtain the worst results, with performance gaps %ranging from
of 1-3\% lower than the proposed model. These results suggest that, in the standard image classification benchmarks used for calibration, our model achieves the best calibration performance, whereas it maintains, or improves, the discriminative power of state-of-the-art classification losses investigated for calibration.

\begin{table}[h!]
  \caption{Classification performance %for different approaches 
  on two popular image classification benchmarks. %Test accuracy (Acc) (\%) is reported. %
  Best method is highlighted in bold, whereas the second best method is underlined. $\Delta$ columns highlight the differences with regard to the best method in each case.
  }
  \label{table:bigClass}
  \centering
  %\small \vspace{1em}
  \footnotesize
  \resizebox{1.0\columnwidth}{!}
  {
  \setlength{\tabcolsep}{2.0pt}
  \begin{tabular}{@{}llccccccccccc@{}}
    \toprule
    %\multicolumn{2}{c}{Part}                   \\
    % \cmidrule(r){1-2}
    \multirow{2}{*}{\textbf{Dataset}}  &     \multirow{2}{*}{\textbf{Model}}  & 
    \multirow{2}{*}{\textbf{CE}} & 
    \multirow{2}{*}{\textbf{ECP}} &
	\multirow{2}{*}{\textbf{LS} } & \multirow{2}{*}{\textbf{FL} } &
	\multirow{2}{*}{\textbf{FLSD}} & 
    \multicolumn{2}{c}{\textbf{Ours (m=0)}} && \multicolumn{2}{c}{\textbf{Ours}} \\
     \cmidrule{8-9} \cmidrule{11-12} 
     &&&&&&& Acc & $\Delta$ && Acc & $\Delta$ \\
    \midrule
    \multirow{2}{*}{Tiny-ImageNet} & R-50 & 65.02 & 64.98 & \textbf{65.78} & 63.09 & 64.09 &  \underline{65.15} & -0.63 &&  64.74 & -1.04 \\
    & R-101 & 65.62 & 65.69 & \textbf{65.87} & 62.97 & 62.96 & 65.72 & -0.15 && \underline{65.81} & -0.06 \\
    \midrule
    \multirow{2}{*}{CIFAR-10} & R-50 & 93.20 & 94.75 &  \underline{94.87} & 94.82 & 94.77 &  94.76 & -0.49 && \textbf{95.25} & +0.38  \\
    & R-101 & 93.33 & 93.35 & 93.23 & 92.42  & 92.38 &  \textbf{95.36} &  +0.23 && \underline{95.13} & -0.23  \\
    % \midrule
    % \rowcolor{LightGray}
    % & ResNet-50 &  73.35 & 90.52 && 73.46 & \underline{90.54} &&  \underline{74.75} & 90.32 && 72.83 & 89.89 && 73.30 & 90.25 && \textbf{74.94} & 90.28 && 74.70  & \textbf{90.66} \\
    % \rowcolor{LightGray}
    % \multirow{-2}{*}{CUB-200-2011}  & ResNet-101 & 73.09 & 90.32 && 73.51 & 90.80 && \underline{74.51} & 89.97 && 72.87 & \textbf{91.11} && 72.59 & \underline{91.01} && 73.92 & 90.09 && \textbf{74.56} & 90.40 \\
    \bottomrule
  \end{tabular}
  }
\end{table}

%In terms of discriminative performance (Table \ref{table:bigClass}), we can observe that the proposed MbLS provides performance at par with LS and CE. On the other hand, FL and its variant FLSD obtain the worst results.....

%Highlight stability of our model wrt to backbone...

% bigClass

\noindent \textbf{Fine-grained image classification.} We now investigate the calibration and discriminative performance on a more complex scenario. In particular, in the previous section we assessed the behaviour of a variety of methods in the scenario of clearly different categories, whereas in this study we include subordinate classes of a common superior class. This setting is arguably more challenging, mostly due to the difficulty of finding informative regions and extracting discriminative features across subcategories. Results from this study are reported in Table \ref{tab:CUB}. In line with previous results, networks trained with hard-encoded labels leads to overconfident networks. Explicitly penalizing low entropy predictions, i.e., ECP \cite{pereyra2017regularizing}, or implicitly with LS results in better calibrated and higher performing models. Nevertheless, if FL and its variant FLSD are used for training, both calibration and classification performances are degraded, leading to the worst results across models. This suggests that, even though FL has been recently shown to work very competitively on the standard benchmarks \cite{mukhoti2020calibrating}, its calibration benefits might vanish on more complex datasets. Last, the network trained with the proposed MbLS method obtains the best calibration and classification performances, with a remarkable gap compared to the existing methods. Note that, for the sake of fairness, the hyperparameters used for all the models, including our method, are the same as the ones employed on the previous section for Tiny-ImageNet.

%It is noteworthy to mention that the hyperparameter in LS (i.e., $\lambda$) and 

\begin{table}[h!]

  \caption{Results on the fine-grained image classification benchmark  \textit{CUB-200-2011} with ResNet-101 as backbone. %We report test ECE and AECE to measure calibration performance, while test accuracy (Acc) is included for classification performance.  
  }
  \vspace{-2mm}
  \label{tab:CUB}
  \centering
  \footnotesize
 
  %\resizebox{0.7\columnwidth}{!}
 % {
  \begin{tabular}{@{}lccc@{}}
    \toprule
    %\multicolumn{2}{c}{Part}                   \\
    % \cmidrule(r){1-2}
    Method    & Acc  & ECE & AECE \\
    \midrule
     CE & 73.09 & 6.75 & 6.65  \\
    % LS ($\alpha=0.05$)    &   & \textbf{34.22} & 13.81 & 3.165  & 3.158 & 0.1384 & 1.413 \\
    ECP \cite{pereyra2017regularizing} & 73.51 & 5.55 & 5.44 \\
    LS \cite{szegedy2016rethinking}   & 74.51  & 5.16 & 5.14 \\
    FL \cite{lin2017focal}   & 72.87 & 8.41 & 8.39 \\
    FLSD \cite{mukhoti2020calibrating}   & 72.59 & 8.54 & 8.53 \\
    \midrule
    % \cmidrule{2-5}
    \textbf{Ours (m=0)} & 73.92 & 5.11 & 5.29 \\
    \textbf{Ours} & \textbf{74.56}  & \textbf{2.78} & \textbf{2.63} \\
    \bottomrule
  \end{tabular}
  %}
  %\vspace{-5mm}
\end{table}

\begin{figure}[h!]
    \centering
    % \vspace{-0.15in}
    \includegraphics[width=0.9\columnwidth]{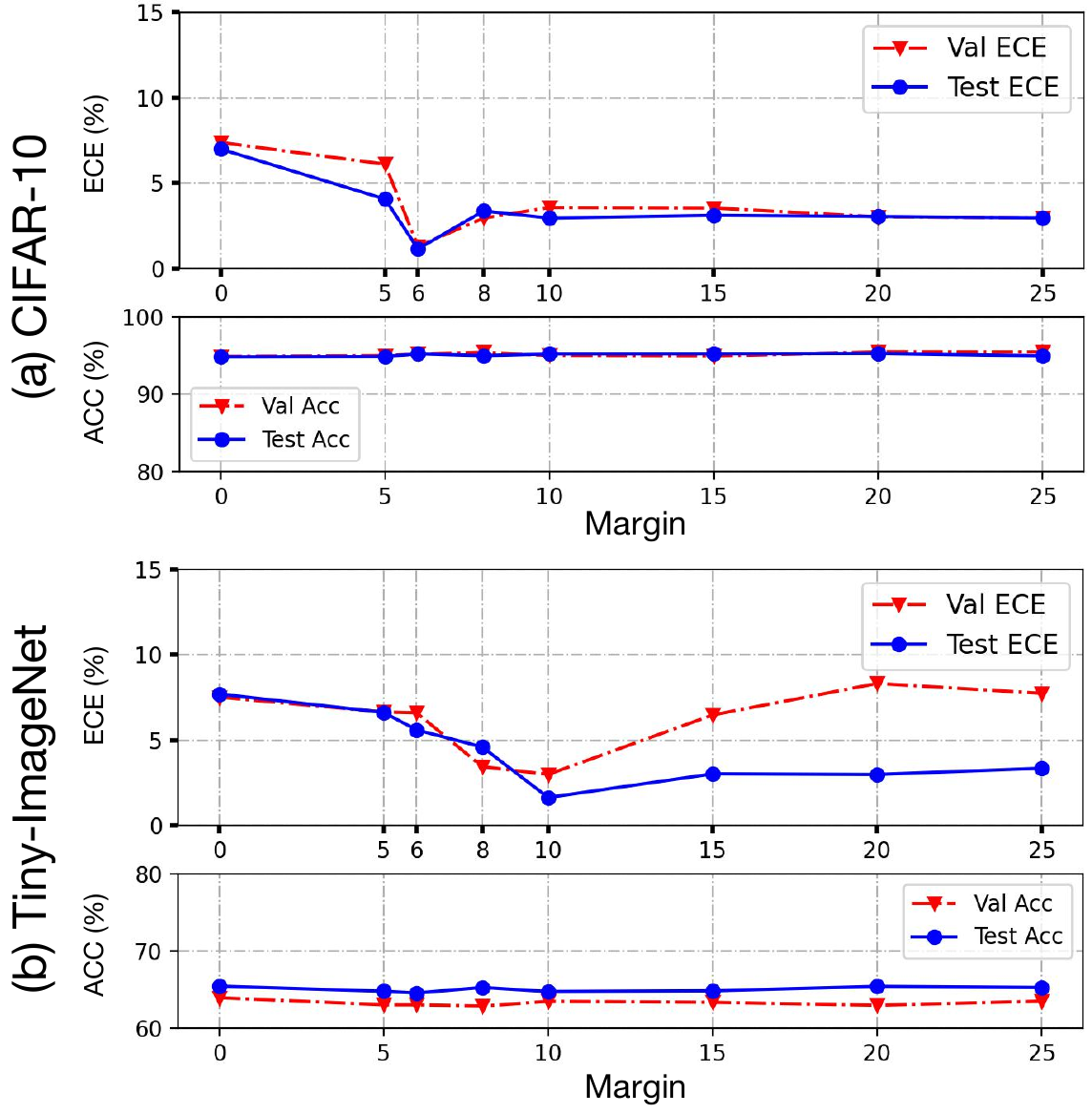}
    % \vspace{-3mm}
    \caption{\textbf{Evaluating the effect of the margin (m).} %in our method.} 
    We present the variation of both ECE and Accuracy on CIFAR-10 (\textit{top}) and Tiny-ImageNet (\textit{bottom}) across different margin values. The network used in this study is ResNet-50 and $\lambda$ in Eq.~\ref{eq:our-l1} is set to 0.1.}
    \vspace{-3mm}
    \label{fig:tiny-margin}
\end{figure}

\noindent \textbf{Effect of the margin $\mathbf{m}$.} In this section, we study the impact of margin $m$ in Eq.~\ref{eq:our-l1}, as depicted on both validation and testing data in Figure~\ref{fig:tiny-margin}. %Noted that we also fixes the relative weight ($\lambda$) to $0.1$ as mentioned in the experimental settings in this experiment, and the network is ResNet-50. 
In particular, we show the evolution of calibration and classification metrics on two datasets, which differ significantly in their input dimensionality. %\textcolor{red}{Our hypothesis is that the choice of the optimal $m$ might depend on the dimensionality of the data.}
The objective of these experiments is to demonstrate the robustness of the method with respect to the margin values, and to show the consistency between the optimal margin values over validation and testing data. Despite the fact that optimal $m$ may vary across different data sets, different choices of $m$ do not affect the performance drastically. Indeed, we can observe that the trend in performance is similar for both datasets, particularly on the testing data. First, imposing a small margin value has a negative impact on the calibration, which might be due to the aggressive gradients resulting from the strong constraint (e.g., $m=0$). Once the optimal $m$ is obtained, larger values result in slightly worst calibrated networks, compared to the best model. Nevertheless, even if we select a network trained with a suboptimal margin, its calibration performance still outperforms state-of-the-art calibration losses. For example, with $m=20$, ECE is equal to 3.05 and 2.99 in CIFAR-10 and Tiny-ImageNet, respectively, whereas LS obtains 2.79 and 3.17, and FL yields 3.90 and 2.96. This demonstrates that our method is capable of bringing, at least, comparable improvements over current literature, even without the need of tuning the value of $m$ over a validation set. 

%Note that $\lambda$ is different....(The lambda is different. In Table 1, the lambda for ours(m=0) is 0.05, now it is 0.1)

%In Fig. \ref{eq:our-l1}(a, b) we show the logits distances on the Tiny-ImageNet training and validation set respectively at different training epoch. The measures of the logits distance are computed by averaging the largest logit distance of each ssample across the data subset. It is demonstrating that our method supports the explicit control of the logit distribution by varying the margin. Fig. \ref{eq:our-l1}(c) gives the comparison of the calibration performance (ECE on test set) with different margin values and the best result is achieve at the point of $10$.

%\subsection{Equivalence with Label Smoothing?}

\begin{table}[h!]
  \caption{Performance of \textit{our method without margin (m=0)} and label smoothing (LS) given equivalent weights on Tiny-ImageNet with ResNet-50.
%   \textbf{Note : $w=\alpha$ in LS (Eq. \ref{eq:ls-kl}) and $w=\lambda$  in Ours (Eq. \ref{eq:our-l1})}. 
  }
  \label{tab:nomargin}
  \centering
  \footnotesize
  \resizebox{0.9\columnwidth}{!}
  {
  \begin{tabular}{@{}c@{}}
       \begin{tabular}{@{}llccccc@{}}
        \toprule
        %\multicolumn{2}{c}{Part}                   \\
        % \cmidrule(r){1-2}
        &\multirow{2}{*}{Method} & \multicolumn{5}{c}{$\alpha$ in LS (Eq.~\ref{eq:ls-kl}) / $\lambda$ in Ours (Eq.~\ref{eq:our-l1})} \\
        \cmidrule{3-7}
         & & $0$ (CE) & $0.05$ & $0.1$ & $0.2$ & $0.3$ \\
        % \cmidrule{2-3} \cmidrule{5-6} \cmidrule{8-9} \cmidrule{10-11}
        % & Acc & ECE && Acc & ECE && Acc & ECE && Acc & ECE \\
        \midrule
         \multirow{2}{*}{ECE} & LS \cite{szegedy2016rethinking}  & 3.73 & 3.17 & 6.53 & 12.05 & 18.04 \\
        & \bf Ours ($m=0$)  & 3.73 & 2.50 & 7.70 & 14.48 & 21.93 \\
        % \cmidrule{2-5}
        \midrule
       \multirow{2}{*}{Acc} & LS \cite{szegedy2016rethinking} & 65.02 & 65.78 & 65.02 &  65.39 & 65.60 \\
        & \bf Ours ($m=0$)  & 65.02 & 65.15 & 65.43 & 65.14 & 66.02 \\
        \bottomrule
       
      \end{tabular}
  \end{tabular}
  }
\end{table}

\noindent \textbf{Equivalence with Label Smoothing.} As presented in the theoretical connections between different losses, Label Smoothing approximates a particular case of the proposed loss when $m$ is equal to $0$.
Table \ref{table:big} and Table \ref{tab:CUB} empirically show that the results of LS and Ours (m=0) are almost consistent for all cases.
It is noted that we follow the best practice in LS by setting equivalent balancing weight of $\lambda$ in Ours (m=0) to $0.05$ in the above experiments.
To further validate empirically this observation, we gradually increase the controlling hyperparameter in both LS ($\alpha$) and our method ($\lambda$). The results are presented in Table \ref{tab:nomargin}. It is seen that by varying the relative trade-off weights between the main cross-entropy and the penalty in both LS and our method we can obtain similar trends and scores, particularly for smaller values of the balancing terms.

\begin{figure*}[h]
    \centering
    % \vspace{-0.15in}
    \includegraphics[width=0.95\textwidth]{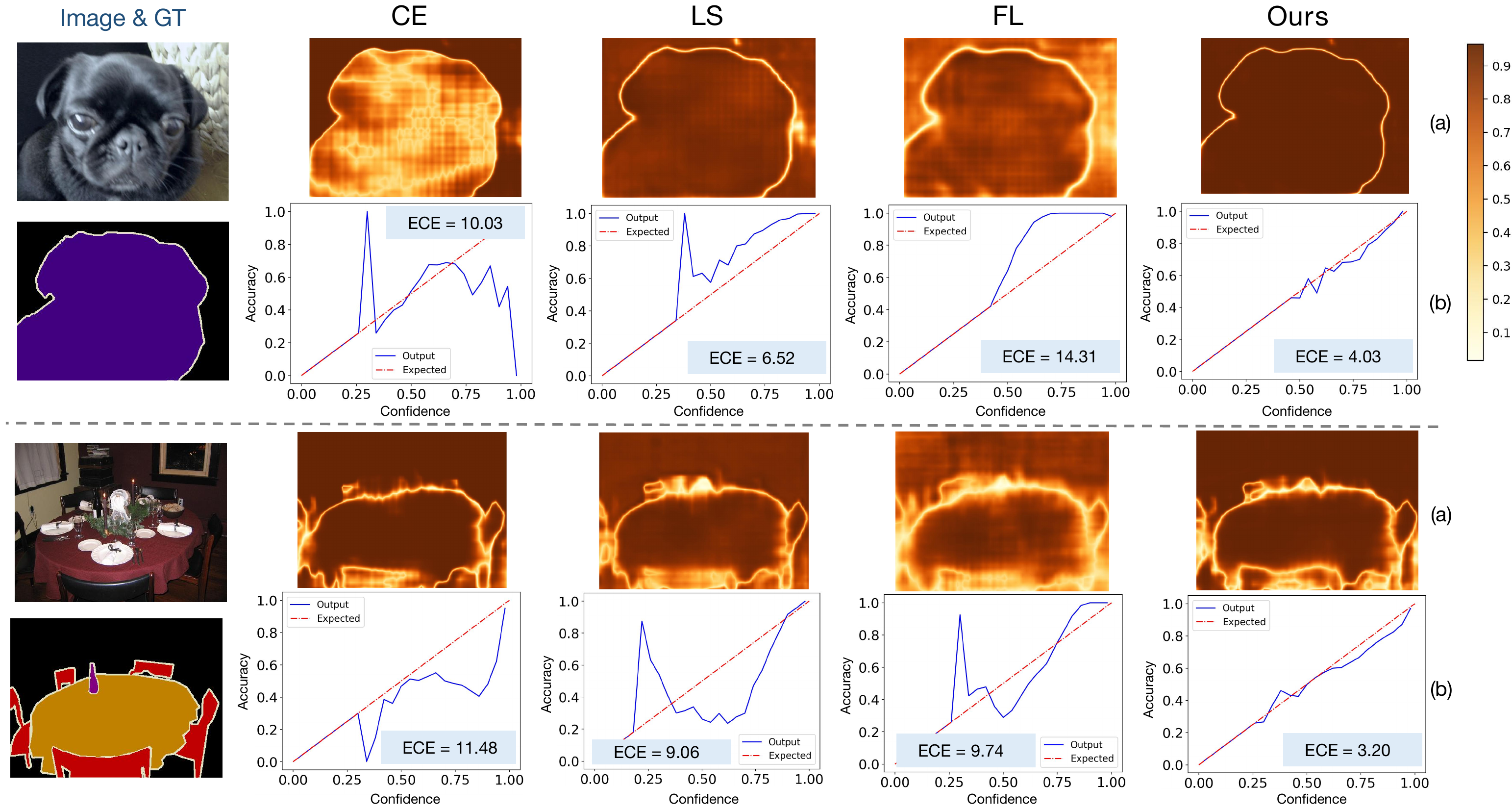}
    \vspace{-3mm}
    \caption{\textbf{Visual results on semantic segmentation.} %in our method.} 
    We present several examples from the qualitative segmentation results on the PASCAL VOC 2012 validation set, showing the superiority by our method, in terms of calibration performance. In the left, we give the original image with ground-truth (GT) mask, then we present the \textbf{ confidence map (a)} and the \textbf{reliability diagram (b)} with the ECE (\%) score for each method. The value of confidence map represent the predicted confidence, i.e., the element of the soft-max probability for the winner class. It is noted that deeper color denotes higher confidence in the map, as shown in the legend at the upper right corner.}
    \vspace{-4mm}
    \label{fig:segment}
\end{figure*}

\noindent \textbf{Results on image segmentation.} Segmentation performances on the popular Pascal VOC dataset are reported in Table \ref{tab:voc}. We can observe that, regardless of the backbone network, the proposed approach leads to the best calibrated and highest performing models, which is consistent with empirical observations in previous experiments. Differences between the proposed method and existing literature are further magnified when ResNet-50, a higher-capacity model, is used as a backbone network. These observations suggest that: \textit{i)} the probability values predicted by our method are a better estimate of the actual likelihood of correctness and \textit{ii)} its calibration performance does not degrade when increasing the model capacity.  %obtains consistently the best calibration metrics.
%Surprisingly, FL...

\begin{table}[h!]
  \caption{Segmentation results on the VOC 2012 validation set. %obtained by DeepLabV3 with different backbones. 
  Best methods are highlighted in bold.}
  \label{tab:voc}
  \centering
  \footnotesize
  \vspace{-2mm}
%   \resizebox{0.7\columnwidth}{!}
  {
  \begin{tabular}{@{}llccc@{}}
    \toprule
    %\multicolumn{2}{c}{Part}                   \\
    % \cmidrule(r){1-2}
    Backbone &  Method    & mIoU  & ECE & AECE \\
    \midrule
    \multirow{5}{*}{ResNet-34} & CE & 68.78 & 8.94 & 8.89  \\
    % LS ($\alpha=0.05$)    &   & \textbf{34.22} & 13.81 & 3.165  & 3.158 & 0.1384 & 1.413 \\
    & ECP \cite{pereyra2017regularizing} & 69.54 & 8.72 & 8.68 \\
    & LS \cite{szegedy2016rethinking}   & 69.71 & 8.11 & 8.47 \\
    & FL \cite{lin2017focal}   & 68.31 & 11.60 & 11.61 \\
    % \cmidrule{2-5}
    & \textbf{Ours} & \textbf{70.24}  & \textbf{7.93} & \textbf{8.00} \\
    \midrule
    \multirow{5}{*}{ResNet-50} & CE     & 70.92 & 8.26 & 8.23  \\
    % LS ($\alpha=0.05$)    &   & \textbf{34.22} & 13.81 & 3.165  & 3.158 & 0.1384 & 1.413 \\
    & ECP \cite{pereyra2017regularizing} & 71.16 & 8.31 & 8.26 \\
    & LS \cite{szegedy2016rethinking}   & 71.00 & 9.35 & 9.95\\
    & FL \cite{lin2017focal}    & 69.99 & 11.44 & 11.43 \\
    % \cmidrule{2-5}
    & \textbf{Ours} & \textbf{71.20}  & \textbf{7.94} & \textbf{7.99} \\
    % & \textbf{Ours} & \textbf{71.05}  & \textbf{7.59} & \textbf{7.70} \\
    
    %\midrule
    %\multirow{5}{*}{ResNet-101} & CE     & 73.25 & 7.27 & 7.23  \\
    % LS ($\alpha=0.05$)    &   & \textbf{34.22} & 13.81 & 3.165  & 3.158 & 0.1384 & 1.413 \\
    %& ECP \cite{pereyra2017regularizing} &  &  &  \\
    %& LS \cite{szegedy2016rethinking}   &  &  & \\
    %& FL \cite{lin2017focal}    &  &  &  \\
    % \cmidrule{2-5}
    %& \textbf{Ours} & 73.86  & 7.27 & 7.27 \\
    \bottomrule
  \end{tabular}
  }
  \vspace{-5mm}
\end{table}

%\noindent \textbf{Visual results on semantic segmentation.}

Furthermore, several visual results from the segmentation task are depicted in Figure~\ref{fig:segment}. In particular, we show the confidence maps (\textit{a}) and the reliability diagrams (\textit{b}) for each method. We can observe that the proposed model provides the best reliability diagrams, as the ECE curves are closer to the diagonal. This indicates that the predicted probabilities are a good estimate of the correctness of the prediction.
% In comparison to the reliability diagrams on image classification, i.e., Figure~\ref{fig:tiny-resnet50}, other methods degrades more in the this more challenging application, while our method shows better robustness and generability.
As for the confidence maps, one may observe several interesting facts. First, %our method tends to yield consistent confidence within the region for each object or class.
%In addition, 
the confidence maps obtained by our method show better edge sharpness, matching the expected property that the model should be less confident at the boundaries, while yielding confident predictions for within-region pixels. In contrast, one could observe that the other methods struggle to provide precise uncertainty estimates, particularly at region boundaries. In addition, in some other cases, existing methods may fail at generating reliable uncertainty regions even within the inner regions of the object (e.g., FL). These visual observations are further supported by the quantitative results in Table \ref{tab:voc}.
More examples are provided in Appendix~\ref{sec:ap:vis} (Figure~\ref{fig:ap:segment}).

%\subsection{Results on text classification. } 

\noindent \textbf{Results on text classification.} We also investigate the calibration of models trained on non-visual pattern recognition tasks, such as text classification, which are evaluated on the 20 Newsgroups dataset. Table \ref{table:NLP} reports the results on this benchmark, which show that the proposed model achieves better discriminative and calibration performance compared to existing works. It is noteworthy to mention that differences are substantial in terms of calibration, suggesting that the proposed approach provides significantly better uncertainty estimates for this task than the competing methods.

\begin{table}[h!]
  \caption{Results on the testing set of the 20 Newsgroups dataset. %, compared to related methods. 
  %Accuracy and \textit{ECE} for each method are reported and 
  Best method is highlighted in bold.}
  \label{table:NLP}
  \vspace{-2mm}
  \centering
  \small 
  \resizebox{1.0\columnwidth}{!}
  {
  \setlength{\tabcolsep}{2.0pt}
  \begin{tabular}{@{}cccccccccccccccccc@{}}
    \toprule
    %\multicolumn{2}{c}{Part}                   \\
    % \cmidrule(r){1-2}
     \multicolumn{2}{c}{\textbf{CE}} && \multicolumn{2}{c}{\textbf{ECP}} \cite{pereyra2017regularizing} &&
			\multicolumn{2}{c}{\textbf{LS} \cite{szegedy2016rethinking}} && \multicolumn{2}{c}{\textbf{FL} \cite{lin2017focal}} &&
			\multicolumn{2}{c}{\textbf{FLSD}} \cite{mukhoti2020calibrating} && 
			\multicolumn{2}{c}{\textbf{Ours}} \\
    %\cmidrule{1-2} \cmidrule{4-5} \cmidrule{7-8} \cmidrule{10-11} \cmidrule{12-13} \cmidrule{14-15} \cmidrule{16-17}
    \midrule
    Acc & ECE && Acc & ECE && Acc & ECE && Acc & ECE && Acc & ECE && Acc & ECE  \\
     67.01 & 22.75 && 66.48 & 22.97 && 67.14 & 8.07 && 66.08 & 10.80 && 65.85 & 10.87 &&  \textbf{67.89} & \textbf{5.40} \\
    \bottomrule
  \end{tabular}
  }
%   \vspace{-5mm}
\end{table}

%\section{Things to modify}
%\begin{itemize}
    %\item Remove Figure 1 and 2 \textcolor{red}{Done}
    %\item remove Acc@5 from table 2 and add one or two columns (maybe one for ours (m=0) and other with ours) highlighting the difference wrt the best method (only for classification) \textcolor{red}{Done}
    %\item Calibration metric in the object boundaries (in segmentation)?
    %\item visual results from segmentation?
    %\item Improve quality of Figure 3. \textcolor{red}{Done}
    %\item 1- CUB in table alone with RN-100 only \textcolor{red}{Done}
    %\item Which are the discriminative results for Table 4? I think providing both gives a better perspective of the performance of each method. 
%\end{itemize}

\vspace{-3mm}

\section{Limitations}

%Maybe discussing about not experiments in domain shift? This section sucks as it might help reviewers to find arguments to criticize our work....

% In this paper, we provided theoretic justifications connecting state-of-the-art calibration losses, which are drawn from a constrained-optimization perspective. Based on these insights, we proposed a novel margin-based label smoothing formulation, which allows to control the margin on the logit distances in an explicit manner.
Despite the superior performance of our method over existing approaches, there exist several limitations in this work. For instance, recent evidences in the literature \cite{ovadia2019can} have demonstrated that simple temperature scaling method does not work well under data/domain distributional shift, and advocate the use of more complex methods that take epistemic uncertainty into account as the shift increases, such as ensembles. Nevertheless, despite these findings, the performances of baselines (i.e., LS \cite{szegedy2016rethinking} or focal loss \cite{lin2017focal}) and the proposed model have not been well investigated in this scenario, which might shed light about potential benefits or drawbacks of these approaches on non-independent and identically distributed (\textit{i.i.d.}) regimes.
% \textcolor{red}{Furthermore, the hyperparameters employed for the existing methods are based on recent findings in the literature \cite{mukhoti2020calibrating,muller2019does}. As several datasets are not investigated in these works (e.g., CUB, PascalVOC or 20 Newsgroups), the set of hyperparameters might not be optimal.}  

% Update the cvpr.cls to do the following automatically.
% For this citation style, keep multiple citations in numerical (not
% chronological) order, so prefer \cite{Alpher03,Alpher02,Authors14} to
% \cite{Alpher02,Alpher03,Authors14}.

\section*{Acknowledgements}

This research work is supported by Prompt Quebec and Compute Canada. Adrian Galdran is funded by the EU's Horizon 2020 R\&I programme under the MSC grant agreement No. 892297.

%%%%%%%%% REFERENCES
% {\small
% \bibliographystyle{ieee_fullname}
% \bibliography{egbib}
% }

% \clearpage

\appendix

\section{Proof}
\label{sec:ap:proof}

Here we provide more details for the proof of Prop.~\ref{prop:ls} in the main text:

\begin{proposition}
\label{ap:prop:ls}
A linear penalty (or a Lagrangian) for constraint $\dd (\llll) = {\mathbf 0}$ is bounded from above and below
by ${\cal D}_\text{KL}\left({\mathbf u} || \s \right )$, up to additive constants:
\begin{align}
{\cal D}_\text{KL}\left({\mathbf u} || \s \right ) - \log(K) \cleq \frac{1}{K}\sum_k (\max_j (l_j) - l_k) \cleq  {\cal D}_\text{KL}\left({\mathbf u} || \s \right ) \nonumber
\end{align}
where $\cleq$ stands for inequality up to an additive constant.
\end{proposition}

\begin{proof}

Given the expression of the KL divergence:
\begin{align}
\label{ap:eq:kl}
{\cal D}_\text{KL}\left({\mathbf u} || \s \right ) = \frac{1}{K} \sum_k \log \left (\frac{1/K}{s_k} \right ) \ceq - \frac{1}{K} \sum_k \log(s_k) \nonumber
\end{align}
where $\ceq$ stands for equality  up to an additive and/or non-negative multiplicative constants and ${\mathbf u}$ is the uniform distribution, and given the definition of softmax function:
\begin{align}
% \label{ap:eq:softmax}
s_{k} = \frac{e^{l_k}}{\sum_j^K e^{l_j}} \nonumber
\end{align}
we have:
\begin{align}
% \label{ap:eq:kl-soft}
{\cal D}_\text{KL}\left({\mathbf u} || \s \right ) &\ceq - \frac{1}{K} \sum_k \log \left (\frac{e^{l_k}}{\sum_j^K e^{l_j}} \right ) \nonumber\\
&= \frac{1}{K} \sum_{k}^K \left ( \log{\sum_j^K e^{l_j}} - l_k \right ) \nonumber
\end{align}
Then, considering the following well-known property of the LogSumExp function:
\begin{align}
    \max_j(l_j) \leq \log{\sum_j^K e^{l_j}} \leq \max_j(l_j) + \log(K) \nonumber
\end{align}
We obtain :
\begin{align}
{\cal D}_\text{KL}\left({\mathbf u} || \s \right ) - \log(K) \cleq \frac{1}{K}\sum_k (\max_j (l_j) - l_k) \cleq  {\cal D}_\text{KL}\left({\mathbf u} || \s \right ) \nonumber
\end{align}

Furthermore, given the definition of the logit distances, i.e., $\dd (\llll) = (\max_j (l_j) - l_k)_{1 \leq k \leq K} \in \mathbb{R}^{K}$, the penalty term, ${\cal D}_\text{KL}\left({\mathbf u} || \s \right )$, imposed by Label Smoothing (LS) is approximately optimizing a linear penalty (or a Lagrangian) for logit distance constraint:
\begin{align}
\dd (\llll) = {\mathbf 0} \nonumber
\end{align}
which encourages equality of all logits.

\end{proof}

\section{Dataset Description and Implementation Details}
\label{sec:ap:data}

In this section, we present the description of all the datasets used in our experiments, as well as the related implementation details.

% \begin{itemize}
\noindent \textbf{CIFAR-10} \cite{krizhevsky2009learning} is an image classification dataset that includes a total of 60,000 images with size $32\times32$, divided equally into 10 classes. In our experiments, we use the standard train/validation/test split containing 45,000/5,000/10,000 images, respectively. During the experiments, we fixed the batch size to 128 and use SGD optimizer with a momentum of $0.9$. The number of training epochs is set to 350, with a multi-step learning rate decay strategy, i.e., learning rate of 0.1 for the first 150 epochs, 0.01 for the next 100 epochs and 0.01 for the last 100 epochs.
Data augmentation techniques like random crops and random horizontal flips are applied on the training set.

\begin{table*}[htb!]
%  \captionsetup{font=footnotesize}
  \caption{ECE for different methods with pre- and post-temperature scaling. Optimal T is indicated in brackets.
  }
  \label{table:temp}
  \centering
  %\small \vspace{1em}
   \footnotesize
  \vspace{-1.em}
  \resizebox{0.8\textwidth}{!}
  {
  \setlength{\tabcolsep}{2.0pt}
  \begin{tabular}{@{}lllccccccccccccccc@{}}
    \toprule
    %\multicolumn{2}{c}{Part}                   \\
    % \cmidrule(r){1-2}
    \multirow{2}{*}{\scriptsize Dataset}  & \multirow{2}{*}{\scriptsize Model}  && \multicolumn{2}{c}{\textbf{CE}} &&
			\multicolumn{2}{c}{\textbf{LS}} && \multicolumn{2}{c}{\textbf{FL}} &&
			\multicolumn{2}{c}{\textbf{FLSD}} &&  \multicolumn{2}{c}{\textbf{Ours}} \\
    % \cmidrule{5-6} \cmidrule{8-9} \cmidrule{11-12} \cmidrule{14-15}
    &  && PreT & PosT && PreT & PosT && PreT & PosT && PreT & PosT && PreT & PosT \\
    \midrule
    \multirow{2}{*}{\scriptsize Tiny-ImageNet} & {\scriptsize R-50} && 3.73 & 1.86 (1.1)  && 3.17 & 1.79 (0.9) && 2.96 & 1.74 (0.9) && 2.91 & 1.74 (0.9) && \textbf{1.64} & \textbf{1.64} (1.0) \\
    & {\scriptsize R-101} && 4.97 & 2.01 (1.2) && 2.20 & 2.20 (1) && 2.55 & 2.22 (0.9) && 4.91 & 1.64 (0.9) && \textbf{1.62} & \textbf{1.62} (1.0) \\
    \midrule
    \multirow{2}{*}{\scriptsize CIFAR-10} & {\scriptsize R-50} && 5.85 & 2.34 (3.9) && 2.79 & 1.75 (0.9) && 3.90 & 1.34 (0.7) && 3.84 & 1.30 (0.7) && \textbf{1.16} & \textbf{1.16} (1.0) \\
    & {\scriptsize R-101} && 5.74 & 2.51 (3.9) && 3.56 & 2.71 (0.9) && 4.60 & 1.24 (1.4) && 4.58 & 1.21 (1.9)  && \textbf{1.38} &  \textbf{1.13} (0.9) \\
     \midrule
    {\scriptsize CUB-200-2011} & {\scriptsize R-101} && 6.75 & 2.00 (1.2) && 5.16 & 3.05 (0.9) && 8.41 & 2.45 (0.8) && 8.54 & 3.61 (3.8)  && \textbf{2.78} & \textbf{1.72} (1.2) \\
    \midrule
    {\scriptsize 20 News} & {\scriptsize GPCN} && 22.75 & 3.01 (3.1) && 8.07 & 3.69 (1.2) && 10.80 & 3.33 (1.4) && 10.87 & 4.10 (1.4)  && \textbf{5.40} & \textbf{2.09} (1.1) \\
    \bottomrule
  \end{tabular}
  }
  \vspace{-1 em}
\end{table*}

\noindent \textbf{Tiny-ImageNet} \cite{deng2009imagenet} is a subset of ImageNet containing 64$\times$64 dimensional images, with 200 classes and 500 images per class in the training set, and 50 images per class in the validation set. %The image dimensions of Tiny-ImageNet are twice that of CIFAR-10 images. 
Following the setting in \cite{mukhoti2020calibrating}, we use 50 samples per class (a total of 10,000 samples) from the training set as a validation set and the original validation set as a test set. 
The batch size is set to $64$. We train for 100 epochs with a learning rate of 0.1 for the first 40 epochs, of 0.01 for the next 20 epochs and of 0.001 for the last 40 epochs.

\noindent \textbf{CUB-200-2011}\cite{WahCUB_200_2011} is the most popular fine-grained benchmarking dataset. As an extended version of the CUB-200 dataset, with roughly double the number of images per class and new part location annotations, it consists of $5,994$ training and $5,794$ test images, belonging to 200 bird species. We augment the images during training, i.e., we resize the images to $256 \times 256$ and then randomly crop patches of $224 \times 224$ from the scaled images or their horizontal flip as inputs.
We initialize the model by pre-trained weights on ImageNet and then train on this dataset for 200 epochs.
The batch size is set to 16 and SGD optimizer is used with a momentum of 0.9.
The learning rate is initialized as 0.1 and decayed by a factor of 0.1 every 80 epochs. Note that, for margin $m$, we used the optimal $m$ found on the validation set of Tiny-ImageNet (we did not use a validation set for CUB-200-2011).
    
\noindent \textbf{PASCAL VOC 2012}\cite{VOC2015} semantic segmentation benchmark contains 20 foreground object classes and one background class. The data is split into $1,464$ images for training, $1,449$ for validation and $1,456$ for testing. Note that the calibration performance on test set is unavailable, as the ground-truth on test set is not publicly released. Therefore, we only report the performances on validation set by using the best hyper-parameters found on the Tiny-ImageNet classification benchmark for all the methods, without any further tuning on the segmentation validation set.
During training, we randomly crop the images to a $512\time512$ resolution, and apply other augmentations such as random horizontal flip, random brightness changes or contrast transformation. 
To train the segmentation model, we employ the popular public library\footnote{\url{https://github.com/qubvel/segmentation_models.pytorch}}, where the encoder is initialized with the weights pre-trained on ImageNet and the decoder is trained from scratch. The batch size is set to 8, and the momentum of the SGD optimizer to 0.9. The learning rate is initialized as 0.01, and decayed by a factor of 0.1 every 40 epochs. 
Finally, the network is trained for 100 epochs.

\noindent \textbf{20 Newsgroups} \cite{lang1995newsweeder} is a popular text classification benchmark, containing 20,000 news articles, which are categorised evenly into 20 different groups based on their content. While some of the groups are significantly related (e.g. rec.motorcycles and rec.autos), other groups are completely unrelated (e.g. sci.space and misc.med). We use the standard train/validation/test split containing 15,098/900/3,999 documents, respectively. To train the Global Pooling Convolutional Network (GPCN) \cite{lin2013network}, we use Glove word embeddings \cite{pennington2014glove}. Adam is used as optimizer with an initial learning rate of 0.001, and beta values equal to 0.9 and 0.999. The training is performed during 100 epochs, with a learning-rate decay by a factor of 0.1 after the first 50 epochs.

\begin{figure*}[h]
    \centering
    % \vspace{-0.15in}
    \includegraphics[width=0.95\textwidth]{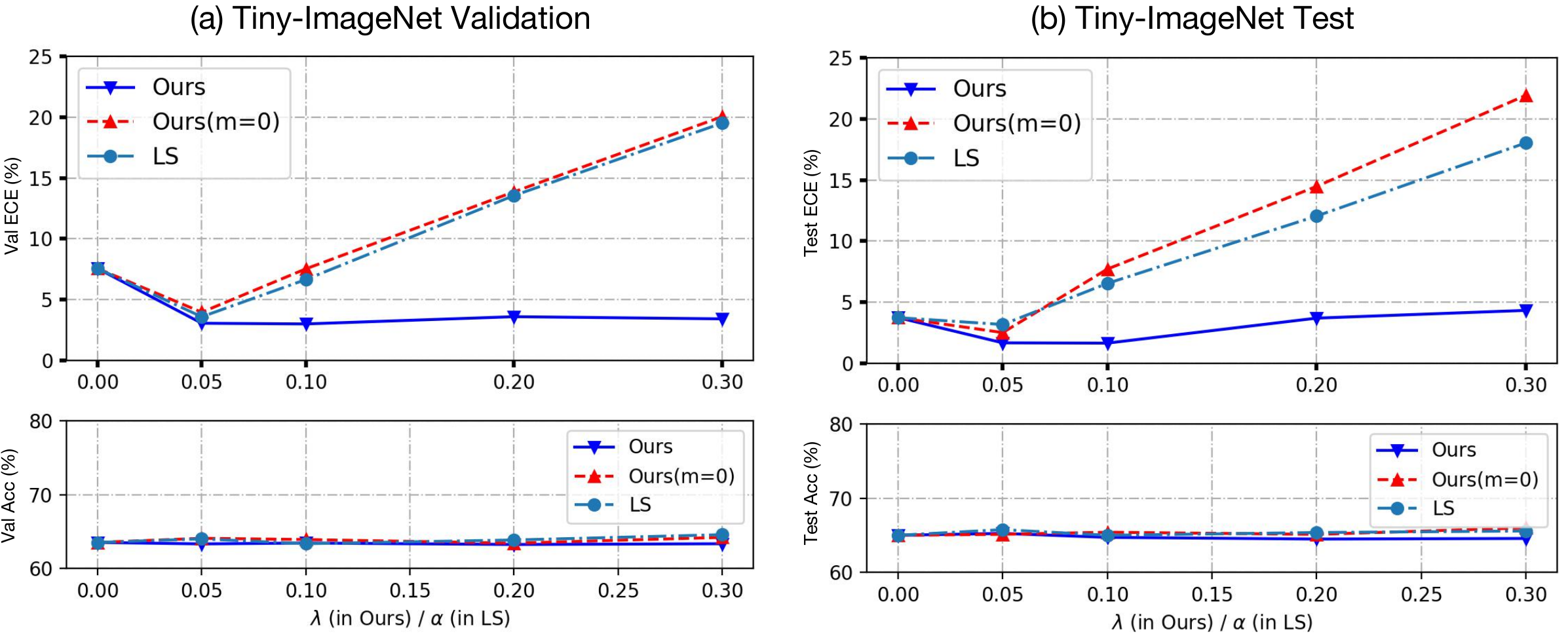}
    \vspace{-2mm}
    \caption{\textbf{Evaluating the effect of the balancing weight.} %in our method.} 
    We present the variation of both ECE and Accuracy on the Tiny-ImageNet validation set (\textit{left}) and on Tiny-ImageNet test set (\textit{right}) using different balancing weight values, i.e., $\lambda$ in our method and $\alpha$ in LS. The network used in this study is ResNet-50.}
    \vspace{-2mm}
    \label{ap:fig:tiny-lambda}
\end{figure*}

\section{Ablation study on the balancing weight}
\label{sec:ap:lambda}

We now investigate the impact of the balancing weight $\lambda$ in our method, and compare it to the effect of $\alpha$ in Label Smoothing (LS), whose results are depicted in Figure~\ref{ap:fig:tiny-lambda}.
In particular, we show the evolution of calibration and classification metrics on Tiny-ImageNet validation and test sets.
One may observe that, unlike LS, our method with margin is more robust with respect to the balancing weight in both subsets. Furthermore, the high similarity in the ECE curves of LS and Ours ($m=0)$ support our theoretical connections stating that LS approximates a particular case of the proposed loss when the margin is equal to 0.

\section{Results with post temperature scaling}
\label{sec:ap:post}

In Table \ref{table:temp}, we compare with the method of applying post temperature scaling (PosT)\cite{guo2017calibration} on the outputs of the CE-trained model.
As this technique is orthogonal to the learning objectives, we also include the results when applying this post-processing to the proposed method.
We can see that the PreT scores obtained by our method outperform the PosT results from CE across all the cases. Furthermore, our method with PosT also achieves the best performance across the datasets and backbones. It is worth noting that the proposed method has optimal temperature values very close to 1 (see Table \ref{table:temp}), indicating that our models are already well calibrated. Note that the results of post-hoc scaling might be highly sensitive to the validation sets and data characteristics.

\begin{figure*}[ht!]
% \begin{wrapfigure}{rt!}{0.5\textwidth}
    \centering
    \includegraphics[width=0.9\textwidth]{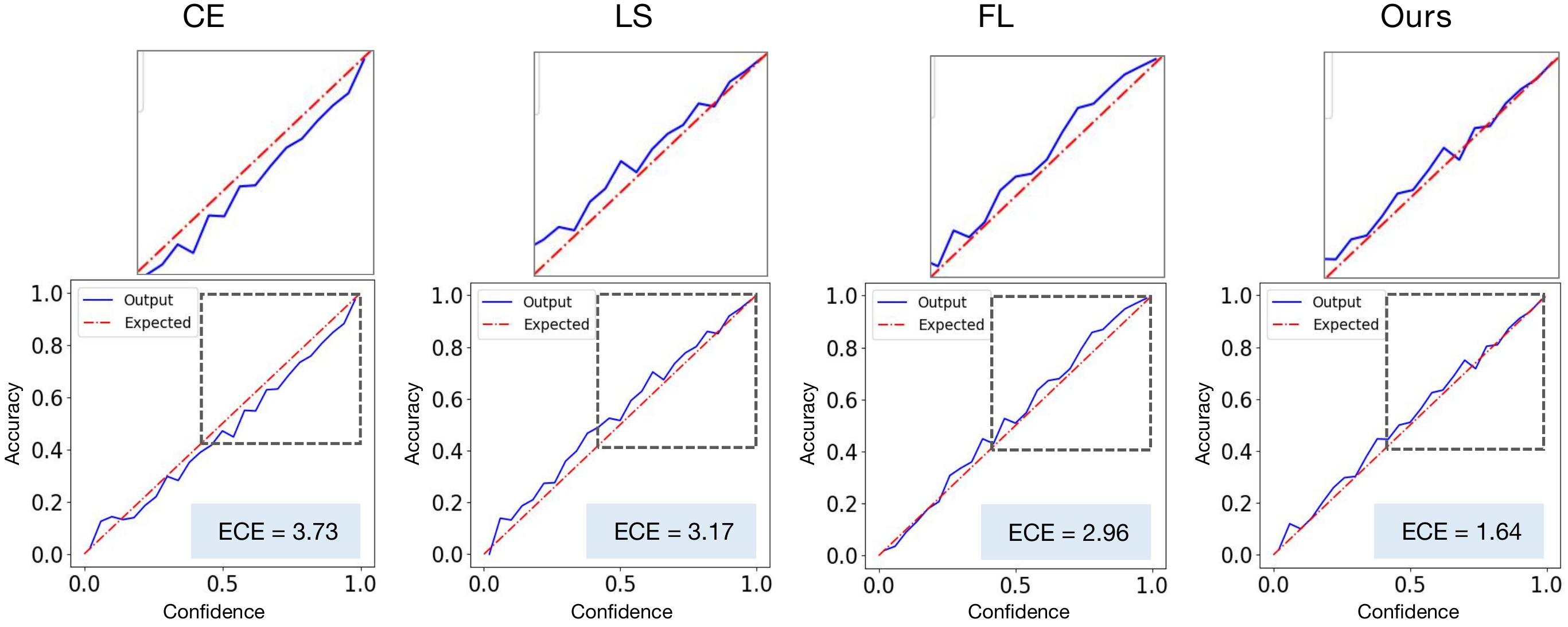}
    \caption{\textbf{Calibration visualizations of ResNet-50 on Tiny-ImageNet.} Reliability diagrams is computed with 25 bins. The zoom-in figures for part of the diagrams are also included, clearly showing the differences.}
    \label{fig:tiny-resnet50}
% \end{wrapfigure}
\end{figure*}

\section{Results with Vision Transformers (ViT)}
\label{sec:ap:vit}

\begin{table}[htb!]
%   \captionsetup{font=footnotesize}
  \caption{Results with Vision Transformer (ViT) model.
 %The fine-grained image classification dataset is shadowed in the table.%
  }
  \label{table:vit}
  \centering
  %\small \vspace{1em}
   \footnotesize
  \vspace{-1.em}
  \resizebox{0.9\columnwidth}{!}
  {
  \setlength{\tabcolsep}{2.0pt}
  \begin{tabular}{@{}llccccccccccccc@{}}
    \toprule
    %\multicolumn{2}{c}{Part}                   \\
    % \cmidrule(r){1-2}
    \multirow{2}{*}{\textbf{Dataset}}  &&  
    % \multicolumn{2}{c}{\textbf{CE}} &&
			\multicolumn{2}{c}{\textbf{LS}} && \multicolumn{2}{c}{\textbf{FL}} && \multicolumn{2}{c}{\textbf{FLSD}} && \multicolumn{2}{c}{\textbf{Ours}} \\
    && Acc & ECE && Acc & ECE && Acc & ECE && Acc & ECE\\
    \midrule
    % Tiny-ImageNet && Res-50 && 3.73 && 5.15 && 1.86 && \textbf{1.64} \\
    % \midrule
    CIFAR-10 && 
    % 98.50 & 0.67 && 
    \textbf{98.57} & 1.39 && 98.49 & 1.20 && 98.55 & 1.13 && \textbf{98.57} & \textbf{0.39} \\
    Tiny-ImageNet && 
    % 90.05 & 2.17 && 
    90.50 & 2.37 && 90.39 & 4.51 && 90.47 & 4.25 && \textbf{90.65} & \textbf{1.26} \\
    \bottomrule
  \end{tabular}
  }
%   \vspace{-1.5 em}
\end{table}

The recent study in \cite{wang2021rethinking} suggests that newer models, such as vision transformers (ViT) \cite{dosovitskiy2020vit}, are better calibrated than older models, such as convolutional neural networks. Inspired by these findings, we further evaluate the performance of the proposed method with ViT, whose results are presented in Table \ref{table:vit}.
In particular, we include the results obtained with a ViT on both CIFAR-10 and Tiny-ImageNet, demonstrating a similar trend, i.e., the proposed approach outperforms other calibration losses. This consolidates the message of this paper and further demonstrates the generalizability of the proposed loss.

\begin{figure*}[h]
    \centering
    % \vspace{-0.15in}
    \includegraphics[width=0.95\textwidth]{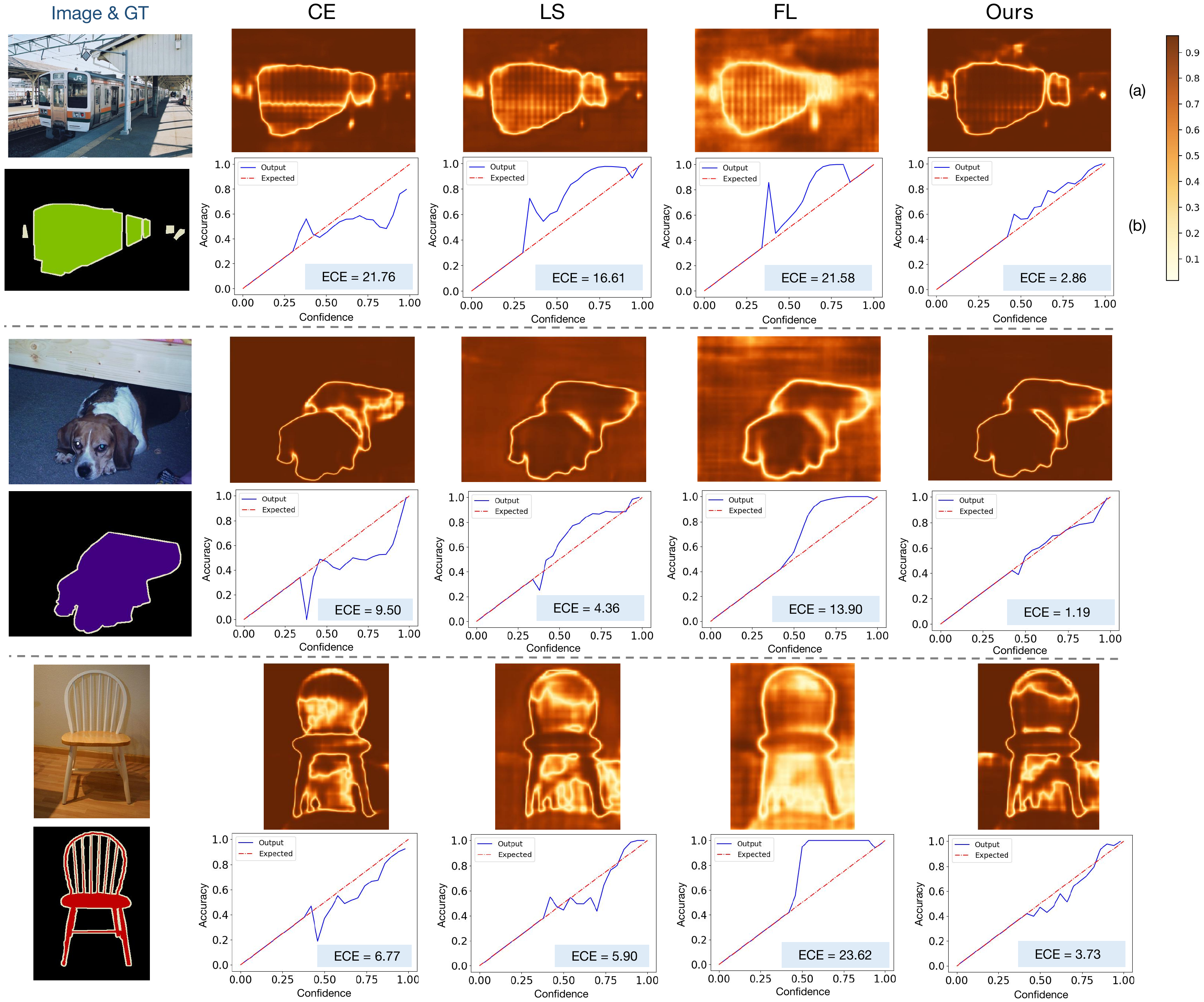}
    \vspace{-3mm}
    \caption{\textbf{Additional visual results on semantic segmentation.} %in our method.} 
    We present additional examples from the qualitative segmentation results on the PASCAL VOC 2012 validation set, showing the superiority by our method, in terms of calibration performance. In the left, we give the original image with ground-truth (GT) mask, then we present the \textbf{ confidence map (a)} and the \textbf{reliability diagram (b)} with the ECE (\%) score for each method. The value of confidence map represent the predicted confidence, i.e., the element of the soft-max probability for the winner class. It is noted that deeper color denotes higher confidence in the map, as shown in the legend at the upper right corner.}
    % \vspace{-2mm}
    \label{fig:ap:segment}
\end{figure*}

\section{Reliability diagram.}
\label{sec:ap:diagram}

We further investigate the calibration behaviour of the proposed model with reliability diagrams, whose results for Tiny-ImageNet with ResNet50 are shown in Figure \ref{fig:tiny-resnet50}. What we expect from a perfectly calibrated model is that its reliability diagram matches the dashed red line, where the output likelihood predicts perfectly the model accuracy. We first observe that the model trained with the standard cross entropy (\textit{first plot}) is overconfident, as its accuracy is mostly below the confidence values. Both state-of-the-art methods (\textit{second and third plots}) reverse this trend, and present reliability diagrams closer to the dashed line, which indicates that models trained with these losses are actually better calibrated. Even though both improve the calibration performance, an interesting observation is that the range accuracy vs confidence where they are better calibrated is indeed the opposite (LS provides better estimates for higher probabilities, whereas FL predictions are better calibrated in a low regime, close to 0). Last, we can observe that the reliability diagram slope provided by our method is much closer to a slope of 1, suggesting that the model is better calibrated. This observation is supported by the quantitative results reported in Section~\ref{sec:exp} of the main text.

\section{Additional visual results on segmentation}
\label{sec:ap:vis}

In Figure~\ref{fig:ap:segment}, we present additional qualitative examples from the VOC segmentation model.
As illustrated by the reliability diagrams (\textit{b}) for different methods, our method achieves the best calibration performance.
Regarding the confidence maps (\textit{a}), the results from the proposed model are also consistent with the fact that uncertainty occurs mainly on the boundary while confidence is higher within and outside the segmentation regions.
Note that all the trends are consistent with the examples shown in Figure~\ref{fig:segment} of the main text.

% \section{LogSumExp and Max functions relation}

%  \begin{figure}[htb]
%  % \begin{wrapfigure}{rt!}{0.5\textwidth}
%      \centering
%      \includegraphics[width=0.5\linewidth]{figures/PlotLogSumExp.png}
%      \caption{Visualizations of the LogSumExp function (\textit{in blue}) and the Max function (\textit{in red}), showing the connection between these two functions for $n=2$. We can easily observe that the Max function is a lower bound of LogSumExp. }
%      \label{fig:logsumexp}
%  % \end{wrapfigure}
%  \end{figure}

% In this section we provide more insights about the LogSumExp function, which is one of the key gradient in the proof of Prop.~\ref{prop:ls}.
% It can be seen as a smoothed version of the Max function. As the max function is not differentiable at points where the maximum is achieved in two different variables, the LogSumExp function is infinitely differentiable everywhere. Figure \ref{fig:logsumexp} depicts the connection between these two functions, showing that the Max function is a lower bound on the LogSumExp.

% %%%%%%%%% REFERENCES
% % \newpage

{\small
\bibliographystyle{ieee_fullname}
\bibliography{egbib}

\begin{thebibliography}{10}\itemsep=-1pt

\bibitem{Bertsekas95}
D.P. Bertsekas.
\newblock {\em Nonlinear Programming}.
\newblock Athena Scientific, Belmont, MA, 1995.

\bibitem{blundell2015weight}
Charles Blundell, Julien Cornebise, Koray Kavukcuoglu, and Daan Wierstra.
\newblock Weight uncertainty in neural network.
\newblock In {\em ICML}, 2015.

\bibitem{chen2017rethinking}
Liang-Chieh Chen, George Papandreou, Florian Schroff, and Hartwig Adam.
\newblock Rethinking atrous convolution for semantic image segmentation.
\newblock In {\em CVPR}, 2017.

\bibitem{deng2009imagenet}
Jia Deng, Wei Dong, Richard Socher, Li-Jia Li, Kai Li, and Li Fei-Fei.
\newblock Imagenet: A large-scale hierarchical image database.
\newblock In {\em CVPR}, 2009.

\bibitem{Ding2021LocalTemp}
Zhipeng Ding, Xu Han, Peirong Liu, and Marc Niethammer.
\newblock Local temperature scaling for probability calibration.
\newblock In {\em ICCV}, 2021.

\bibitem{dosovitskiy2020vit}
Alexey Dosovitskiy, Lucas Beyer, Alexander Kolesnikov, Dirk Weissenborn,
  Xiaohua Zhai, Thomas Unterthiner, Mostafa Dehghani, Matthias Minderer, Georg
  Heigold, Sylvain Gelly, Jakob Uszkoreit, and Neil Houlsby.
\newblock An image is worth 16x16 words: Transformers for image recognition at
  scale.
\newblock In {\em ICLR}, 2021.

\bibitem{VOC2015}
Mark Everingham, S.~M. Eslami, Luc Gool, Christopher~K. Williams, John Winn,
  and Andrew Zisserman.
\newblock The pascal visual object classes challenge: A retrospective.
\newblock {\em IJCV}, 111(1):98–136, 2015.

\bibitem{gal2016dropout}
Yarin Gal and Zoubin Ghahramani.
\newblock Dropout as a bayesian approximation: Representing model uncertainty
  in deep learning.
\newblock In {\em ICML}, 2016.

\bibitem{guo2017calibration}
Chuan Guo, Geoff Pleiss, Yu Sun, and Kilian~Q Weinberger.
\newblock On calibration of modern neural networks.
\newblock In {\em ICML}, 2017.

\bibitem{he2016deep}
Kaiming He, Xiangyu Zhang, Shaoqing Ren, and Jian Sun.
\newblock Deep residual learning for image recognition.
\newblock In {\em CVPR}, 2016.

\bibitem{hernandez2015probabilistic}
Jos{\'e}~Miguel Hern{\'a}ndez-Lobato and Ryan Adams.
\newblock Probabilistic backpropagation for scalable learning of bayesian
  neural networks.
\newblock In {\em ICML}, 2015.

\bibitem{krizhevsky2009learning}
Alex Krizhevsky and Geoffrey Hinton.
\newblock Learning multiple layers of features from tiny images.
\newblock Technical report, University of Toronto, 2009.

\bibitem{krizhevsky2012imagenet}
Alex Krizhevsky, Ilya Sutskever, and Geoffrey~E Hinton.
\newblock Imagenet classification with deep convolutional neural networks.
\newblock In {\em NeurIPS}, 2012.

\bibitem{lakshminarayanan2016simple}
Balaji Lakshminarayanan, Alexander Pritzel, and Charles Blundell.
\newblock Simple and scalable predictive uncertainty estimation using deep
  ensembles.
\newblock In {\em NeurIPS}, 2017.

\bibitem{lang1995newsweeder}
Ken Lang.
\newblock Newsweeder: Learning to filter netnews.
\newblock In {\em ICML}, 1995.

\bibitem{larrazabal2021orthogonal}
Agostina~J Larrazabal, C{\'e}sar Mart{\'\i}nez, Jose Dolz, and Enzo Ferrante.
\newblock Orthogonal ensemble networks for biomedical image segmentation.
\newblock In {\em MICCAI}, 2021.

\bibitem{lin2013network}
Min Lin, Qiang Chen, and Shuicheng Yan.
\newblock Network in network.
\newblock In {\em ICML}, 2014.

\bibitem{lin2017focal}
Tsung-Yi Lin, Priya Goyal, Ross Girshick, Kaiming He, and Piotr Doll{\'a}r.
\newblock Focal loss for dense object detection.
\newblock In {\em CVPR}, 2017.

\bibitem{louizos2016structured}
Christos Louizos and Max Welling.
\newblock Structured and efficient variational deep learning with matrix
  gaussian posteriors.
\newblock In {\em ICML}, 2016.

\bibitem{lukasik2020does}
Michal Lukasik, Srinadh Bhojanapalli, Aditya Menon, and Sanjiv Kumar.
\newblock Does label smoothing mitigate label noise?
\newblock In {\em ICML}, 2020.

\bibitem{Ma2021postrank}
Xingchen Ma and Matthew~B. Blaschko.
\newblock Meta-cal: Well-controlled post-hoc calibration by ranking.
\newblock In {\em ICML}, 2021.

\bibitem{mukhoti2020calibrating}
Jishnu Mukhoti, Viveka Kulharia, Amartya Sanyal, Stuart Golodetz, Philip~HS
  Torr, and Puneet~K Dokania.
\newblock Calibrating deep neural networks using focal loss.
\newblock In {\em NeurIPS}, 2020.

\bibitem{muller2019does}
Rafael M{\"u}ller, Simon Kornblith, and Geoffrey Hinton.
\newblock When does label smoothing help?
\newblock In {\em NeurIPS}, 2019.

\bibitem{naeini2015ece}
Mahdi~Pakdaman Naeini, Gregory~F. Cooper, and Milos Hauskrecht.
\newblock Obtaining well calibrated probabilities using bayesian binning.
\newblock In {\em AAAI}, 2015.

\bibitem{ovadia2019can}
Yaniv Ovadia, Emily Fertig, Jie Ren, Zachary Nado, David Sculley, Sebastian
  Nowozin, Joshua~V Dillon, Balaji Lakshminarayanan, and Jasper Snoek.
\newblock Can you trust your model's uncertainty? evaluating predictive
  uncertainty under dataset shift.
\newblock In {\em NeurIPS}, 2019.

\bibitem{pennington2014glove}
Jeffrey Pennington, Richard Socher, and Christopher~D Manning.
\newblock Glove: Global vectors for word representation.
\newblock In {\em EMNLP}, 2014.

\bibitem{pereyra2017regularizing}
Gabriel Pereyra, George Tucker, Jan Chorowski, {\L}ukasz Kaiser, and Geoffrey
  Hinton.
\newblock Regularizing neural networks by penalizing confident output
  distributions.
\newblock In {\em ICLR}, 2017.

\bibitem{platt1999probabilistic}
John Platt et~al.
\newblock Probabilistic outputs for support vector machines and comparisons to
  regularized likelihood methods.
\newblock {\em Advances in large margin classifiers}, 10(3):61--74, 1999.

\bibitem{szegedy2016rethinking}
Christian Szegedy, Vincent Vanhoucke, Sergey Ioffe, Jon Shlens, and Zbigniew
  Wojna.
\newblock Rethinking the inception architecture for computer vision.
\newblock In {\em CVPR}, 2016.

\bibitem{Tomani2021Posthoc}
Christian Tomani, Sebastian Gruber, Muhammed~Ebrar Erdem, Daniel Cremers, and
  Florian Buettner.
\newblock Post-hoc uncertainty calibration for domain drift scenarios.
\newblock In {\em CVPR}, 2021.

\bibitem{WahCUB_200_2011}
C. Wah, S. Branson, P. Welinder, P. Perona, and S. Belongie.
\newblock {The Caltech-UCSD Birds-200-2011 Dataset}.
\newblock Technical Report CNS-TR-2011-001, California Institute of Technology,
  2011.

\bibitem{wang2021rethinking}
Deng-Bao Wang, Lei Feng, and Min-Ling Zhang.
\newblock Rethinking calibration of deep neural networks: Do not be afraid of
  overconfidence.
\newblock In {\em NeurIPS}, 2021.

\bibitem{wenzel2020hyperparameter}
Florian Wenzel, Jasper Snoek, Dustin Tran, and Rodolphe Jenatton.
\newblock Hyperparameter ensembles for robustness and uncertainty
  quantification.
\newblock In {\em NeurIPS}, 2020.

\bibitem{xie2016disturblabel}
Lingxi Xie, Jingdong Wang, Zhen Wei, Meng Wang, and Qi Tian.
\newblock Disturblabel: Regularizing cnn on the loss layer.
\newblock In {\em CVPR}, 2016.

\bibitem{zhang2020mix}
Jize Zhang, Bhavya Kailkhura, and T Han.
\newblock Mix-n-match: Ensemble and compositional methods for uncertainty
  calibration in deep learning.
\newblock In {\em ICML}, 2020.

\bibitem{zhang2019confidence}
Zhilu Zhang, Adrian~V Dalca, and Mert~R Sabuncu.
\newblock Confidence calibration for convolutional neural networks using
  structured dropout.
\newblock {\em arXiv preprint arXiv:1906.09551}, 2019.

\end{thebibliography}
}

\end{document}